%% file: main.tex
\documentclass{article}

\usepackage{arxiv}
\usepackage[numbers]{natbib}

\usepackage{amsmath}
\usepackage[utf8]{inputenc} 
\usepackage[T1]{fontenc}    
\usepackage{hyperref}       
\usepackage{url}            
\usepackage{booktabs}       
\usepackage{amsmath,amsthm,amssymb,amsfonts, bm}      
\usepackage{nicefrac}       
\usepackage{microtype}      
\usepackage{xcolor}         
\usepackage{svg}
\usepackage{lipsum}

\input{macros.tex}

\title{Robust Best-arm Identification in Linear Bandits}

\author{ Wei Wang \\
	University College London\\
	\href{mailto:ucabww2@ucl.ac.uk}{\texttt{ucabww2@ucl.ac.uk}} \\
	\And
	Sattar Vakili  \\
	MediaTek Research\\
	\href{mailto:sattar.vakili@mtkresearch.com}{\texttt{sattar.vakili@mtkresearch.com}} \\
        \And 
        Ilija Bogunovic  \\
	University College London\\
	\href{mailto:i.bogunovic@ucl.ac.uk}{\texttt{i.bogunovic@ucl.ac.uk}} \\
}

\begin{document}

\maketitle

\input{01-abstract}


\input{1-introduction}
\input{2-relatedwork}

\vspace{-2pt}
\section{Problem Statement}
\vspace{-2pt}
\label{sec:statement}
We consider the \emph{robust} bandit problem where $\cX \subseteq \R^d$ denotes the finite set of arms of the learner, 
while $\cY(x) \subseteq \mathbb{R}^d$ denotes the finite set of possible perturbations for each $x \in \cX$. We assume that both $\cX$ and corresponding $\cY(x)$ for each $x \in \cX$ are known input sets.\looseness=-1 

Given an unknown parameter $\theta \in \R^d$, we consider a linear reward setting, where at each round $n$, the learner can choose both an arm $x_n\in \cX$ and an adversarial arm $y_n \in \cY(x_n)$ to receive a noisy reward: 
\begin{equation} \label{eq:observation_model}
    r_n = (x_n  - y_n)^\top \theta + \eta_n, 
\end{equation}
\noindent where $\eta_n$ is independent, $R$-sub-Gaussian. We often use the notation $z = x - y$ and $\cZ = \lbrace x-y: \forall x \in \cX, \forall y \in \cY(x)\rbrace$, and assume that $\|z\|_2\leq L$. We denote 
\begin{equation} \label{eq:best_robust_arm}
    x^*=\argmax_{x \in \cX}\min_{y \in \Yx}(x-y)^\top \theta,
\end{equation}
as the best robust arm in $\cX$. 
We assume the best robust arm is unique in the set $\mathcal{X}$, that is, the \textit{robust value gap} $\Delta_r(x^*,x) > 0$ for any $x \in \cX\setminus \{x^*\}$ where
\begin{equation}
    \label{eq:robust_gap}
    \Delta_r (x,x') = \min_{y \in \Yx}(x-y)^\top \theta - \min_{y' \in \cY(x')}(x'-y')^\top \theta.
\end{equation}
In comparison to the standard non-robust problem where the corresponding suboptimality gap is given by $\Delta (x,x') = (x-x')^\top \theta$, the above robust gap calculates the gap after posing the most adversarial actions on the arms and can be expressed as: 
\begin{align}
    \Delta_r (x,x') = \Delta (x,x') - \max_{y \in \mathcal{Y}(x)}y^\top \theta + \min_{y' \in \mathcal{Y}(x')}y'^\top \theta.
\end{align}
Based on the formulation presented above, it is evident that the robust gap $\Delta _r$ depends on the adversarial space $\mathcal{Y}(\cdot)$ of each arm and can either be larger or smaller than the standard gap.

We introduce the following robust best-arm identification problem (RBAI). Let $\hx_n$ be the estimated robust best arm returned by a bandit algorithm after $n$ steps. 
Our research focuses on the $\delta$-PAC setting. We aim to design an allocation strategy and a stopping criterion given $\delta \in (0,1)$ such that when the algorithm stops, the returned arm $\hat{x}_n$ satisfies $\mathbb{P}(\hat{x}_n = x^*) \geq 1-\delta,$ within the smallest number of steps (samples) $n$ as possible.\looseness=-1

\input{3-lowerbound}

\input{4-algorithms}

\vspace{-1ex}
\section{Experiments}
\vspace{-1ex}
\label{sec:experiments}
\input{6-experiment}
\vspace{-1ex}
\section{Conclusions} 
\vspace{-1ex}
Our work aimed to develop robust algorithms for best-arm identification with linear rewards. We presented novel methods based on static-G allocation and adaptive gap elimination sampling. These algorithms demonstrate effectiveness in identifying the robust best arm and achieving optimal sample complexity rates. Through our experiments, we validated the utility of our approach in determining the suitable bolus insulin amount for patients, while considering inaccuracies. In terms of future research, we propose the exploration of algorithms incorporating adaptive confidence bounds \cite{abbasi2011improved}. We believe that our work makes a substantial contribution towards bridging the sim-to-real gap when working with simulators.\looseness=-1 

\bibliographystyle{plainnat}
\bibliography{main}

\input{7A-Appendix-A.tex}

\end{document}

%% file: macros.tex
\usepackage{answers}
\usepackage{setspace}
\usepackage{graphicx}
\usepackage{enumitem}
\usepackage{multicol}
\usepackage{mathrsfs}
\usepackage{amsmath,amsthm,amssymb,amsfonts, bm}

\usepackage{xspace, mathtools}
\usepackage{algorithm}
\usepackage{algorithmic}
\usepackage{dsfont}
\usepackage{bbm}

\DeclareMathOperator*{\argmax}{arg\,max}

\DeclarePairedDelimiter\abs{\lvert}{\rvert}%
\DeclarePairedDelimiter\norm{\lVert}{\rVert}%
\makeatletter
\let\oldabs\abs
\def\abs{\@ifstar{\oldabs}{\oldabs*}}
\let\oldnorm\norm
\def\norm{\@ifstar{\oldnorm}{\oldnorm*}}
\makeatother

\makeatletter
\newcommand{\pushright}[1]{\ifmeasuring@#1\else\omit\hfill$\displaystyle#1$\fi\ignorespaces}
\newcommand{\pushleft}[1]{\ifmeasuring@#1\else\omit$\displaystyle#1$\hfill\fi\ignorespaces}
\makeatother

\newcommand{\R}{\mathbb{R}}

\def\cX{{\mathcal{X}}}
\def\cY{{\mathcal{Y}}}
\def\cZ{{\mathcal{Z}}}

\usepackage{thmtools}
\usepackage{thm-restate}

\theoremstyle{plain}

\declaretheorem[name=Lemma]{lemma}
\declaretheorem[name=Proposition]{proposition}

\newtheorem{definition}{Definition}
\usepackage{colortbl}

\newcommand{\Yx}{\mathcal{Y}(x)}

\newcommand{\hx}{\hat{x}}
\newcommand{\Y}{\mathcal{Y}}
\newcommand{\X}{\mathcal{X}}

\usepackage{graphicx}
\usepackage{grffile}
\usepackage{multirow}
\usepackage{subfigure}

\usepackage[capitalize]{cleveref}

%% file: 01-abstract.tex
\begin{abstract}
    We study the robust best-arm identification problem (RBAI) in the case of linear rewards. The primary objective is to identify a near-optimal robust arm, which involves selecting arms at every round and assessing their robustness by exploring potential adversarial actions. This approach is particularly relevant when utilizing a simulator and seeking to identify a robust solution for real-world transfer. To this end, we present an instance-dependent lower bound for the robust best-arm identification problem with linear rewards. Furthermore, we propose both static and adaptive bandit algorithms that achieve sample complexity that matches the lower bound. In synthetic experiments, our algorithms effectively identify the best robust arm and perform similarly to the oracle strategy. As an application, we examine diabetes care and the process of learning insulin dose recommendations that are robust with respect to inaccuracies in standard calculators. Our algorithms prove to be effective in identifying robust dosage values across various age ranges of patients.\looseness=-1
\end{abstract}

%% file: 1-introduction.tex
\vspace{-2pt}
\section{Introduction}
\vspace{-2pt}

In various real-world applications, such as drug discovery, clinical trials, or patient dose finding, the goal is to identify the optimal solution from a minimal number of expensive trials (samples), often by utilizing an expensive simulator. To speed up the experimental design process, numerous studies have framed this as the \emph{bandit best-arm identification} problem (\citet{audibert2010best, jamieson2014best, soare2014best}), which involves learning about the underlying process and identifying the best solution (arm) with minimal interactions.

The problem of identifying the best arm in non-robust situations has been extensively studied due to its connection to real-world problems. However, the best arm may not necessarily be found in a uniformly high-reward neighborhood. Consequently, the arm's performance can be greatly affected by unexpected changes, such as implementation errors or a gap between simulation and reality. 
For example, in diabetes care, determining the optimal insulin dosage for a patient poses a significant challenge since clinical experimentation with real patients is risky and ethically challenging (\citep{chen2021ethical, vayena2018machine}). Instead, expensive simulators and calculators are used to determine the appropriate dosage, which may not take into account certain patient characteristics even when fine-tuned (\cite{demirel2021safe}). Consequently, determining how to harness the capabilities of the simulator to provide recommendations for a safe, reliable, and efficient dosage in the real world remains a vital challenge, particularly in light of these inaccuracies.\looseness=-1

To address the previous challenges, we formulate a \emph{robust best-arm identification (RBAI)} problem. In particular, during each round, the learner chooses an arm and its corresponding adversarial action and then receives a noisy reward that corresponds to the chosen pair (see \Cref{fig:enter-label}). The objective is to efficiently discover the most resilient arm, meaning the one that yields the highest possible mean reward under the worst-case scenario. 
This amounts to learning about high-reward arms but also assessing their robustness by searching through possible adversarial actions. We study the above-mentioned problem in the classical setting of \emph{linear} reward models. 
Our work provides first instance-dependent lower bounds for the robust best arm identification setting, and an algorithm that matches these up to some logarithmic factors. \looseness=-1


\begin{figure*}
    \begin{multicols}{2}
            \centering
            \includegraphics[width=\textwidth, height=4cm]{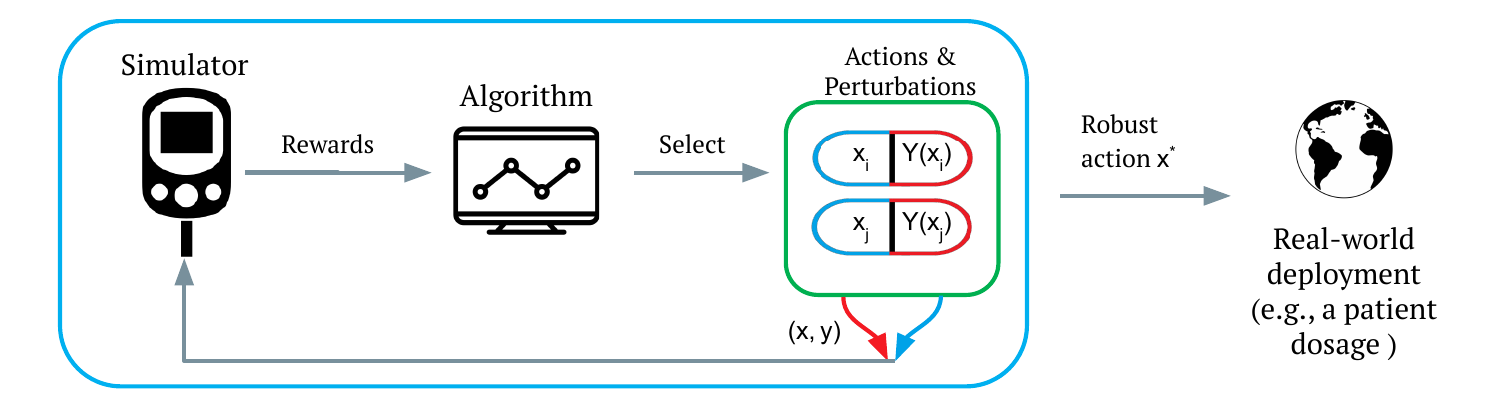}        
    \end{multicols}
    \vspace{-2ex}
    \caption{{An illustration of the robust best-arm identification setting.}}\label{fig:enter-label}
    \vspace{-2ex}
\end{figure*}

%% file: 2-relatedwork.tex
\textbf{Related Work.} Bandit problems with linear rewards have been extensively studied in prior works such as \cite{dani2008stochastic, rusmevichientong2010linearly,agrawal2013thompson, gabillon2012best, degenne2020gamification}. 
Our problem formulation as a linear multi-armed bandit best-arm identification problem \cite{audibert2010best} is similar to \citet{soare2014best}. While \cite{soare2014best} explored a similar problem in the standard non-robust setting, our emphasis lies in the robust setting, specifically in the discovery of the best-robust arms.
Furthermore, \citet{fiez2019sequential} consider a more general transductive setting that extends the concept of linear bandits. In their study, they present instance-dependent lower bounds specifically tailored for the transductive setting. Additionally, they propose an algorithm called RAGE, which achieves performance that closely aligns with these lower bounds. Our robust setting is different from the transductive one since our worst-performing arms are dependent on unknown reward parameters, making it impossible to construct predefined transductive arms.

In addition to addressing robustness, a significant body of prior research in the bandit problem has been dedicated to learning under various constraints for safety purposes. For instance, \cite{amani2019linear}, \cite{moradipari2021safe} and \cite{demirel2021safe} investigate the linear bandit setting, where the constraint function is unknown, and the exploration procedure must adhere to the constraints. \cite{pacchiano2021stochastic} require that the expectation of cost satisfies a given constraint value, while \cite{khezeli2020safe} focus on achieving an expected reward above a specific threshold. \citet{lindner2022interactively} introduce linear constrained BAI problems with unknown constraints, aiming to identify the optimal arm. \citet{wang2022best} also consider a BAI problem with linear constraints. However, in contrast to \cite{lindner2022interactively}, their work assumes unknown rewards and primarily concentrates on linear constraints. \looseness=-1

The linear bandit problem has seen exploration of various notions of robustness. Notably, \cite{bogunovic2021stochastic} and \cite{ding2022robust} focus on the standard linear bandit problem, aiming to mitigate adversarial corruptions in reward observations \cite{bogunovic2021stochastic, ding2022robust}. While \cite{bogunovic2021stochastic} specifically consider rewards attacks, \cite{ding2022robust} extend their investigation to encompass both rewards attacks and context attacks. Recently, \cite{he2022nearly} introduced a computationally efficient algorithm for the linear contextual bandit problem with rewards attacks \cite{he2022nearly}. Moreover, \cite{alieva2021robust} tackle the linear best-arm identification problem with rewards attacks in the fixed-budget setting, proposing a robust pure-exploration algorithm. 
In contrast to these works, this paper solely focuses on standard noisy observations. The primary objective is to identify the best-robust arm while operating under standard noise assumptions. Our robust linear bandit formulation differs significantly from the corrupted and heavy-tailed noise bandits aforementioned. Specifically, the concept of robustness in those works primarily revolves around dealing with (corrupted) observations, whereas our primary focus is on addressing action perturbations. This key distinction sets our approach apart from these works.\looseness=-1

Furthermore, the same problem as ours is considered within the broader framework of kernelized bandits in \citet{bogunovic2018adversarially}. In this setting, the returned arm may be subject to perturbations from an adversary, and the objective is to maintain a high function value even after such perturbations occur. The authors introduce the StableOPT algorithm, which leverages a combination of upper and lower confidence bound strategies when sampling arms and accounting for adversarial actions. However, the guarantees provided in their work adopt a worst-case perspective. In contrast, our algorithms employ distinct strategies rooted in experimental design techniques, and we provide instance-dependent sample complexity guarantees.\looseness=-1

\textbf{Main contributions.} 
Our objectives are twofold: firstly, to define the robust linear bandit problem, and secondly, to determine the instance-optimal sample complexity for this problem. \Cref{sec:statement} of our paper presents instance-dependent lower bounds for the robust bandit problem, which are distinct from the lower bounds established for standard linear bandits. Our bounds coincide with the current lower bound for linear bandits in the specific situation of singleton adversarial sets.
\Cref{sec:algorithms} presents algorithms for robust linear bandits and establishes the corresponding sample complexity results (\Cref{thm:static} and \Cref{thm:robust_rage}). Our analysis demonstrates that the sample complexity we obtain matches the lower bound, up to some logarithmic factors. In \Cref{sec:experiments}, we present our experiments that demonstrate the empirical advantage of our theoretically superior algorithm in two problem scenarios. We also illustrate the practical value of our approach in the context of the diabetes care problem, where our algorithm recommends robust solutions that outperform classical insulin calculators, thereby showcasing its robust performance. \looseness=-1

%% file: 3-lowerbound.tex
\vspace{-2pt}
\subsection{Lower Bound}
\vspace{-2pt}










In this section, we present a lower bound for the complexity of the robust best-arm identification problem. This lower bound indicates the minimum number of samples required to distinguish the best robust arm from another arm that is the closest in terms of its robust value.


\begin{restatable}[RBAI Lower Bound]{theorem}{lowerbound}
    For any adversarial linear bandit environment $\nu = (\mathcal{X},\left \{ \mathcal{Y}(x)  \right \}_{x \in \mathcal{X}},\theta )$, there exists an alternative environment $\nu' = (\mathcal{X},\left \{ \mathcal{Y}(x)  \right \}_{x \in \mathcal{X}},\theta' )$ having the same input sets $\mathcal{X}$ and $\left \{ \mathcal{Y}(x)  \right \}_{x \in \mathcal{X}}$ but a different best robust arm, such that the number of pulls $\tau$ needed by any $\delta$-PAC static allocation strategy to distinguish between the two problems is such that
    \begin{align}
          \mathbb{E}[\tau] &\geq 
          C_{\delta}\max_{y \in \mathcal{Y}(x^*)}\max _{x'\in \cX\setminus x^*}\min_{y'\in\cY(x')}\frac{\|x^* - y - (x' - y') \|^2_{A_{\lambda}^{-1}}}{\max\{\Delta (x^*, y, x',y'),0\}^{2}},
        \label{eq:lowerbound-eq}
    \end{align}
    where $C_{\delta} = 2\log (1 / 2 \delta)$, 
    $\lambda$ is a probability distribution over arms which the allocation strategy follows (i.e., $\lambda(x,y)$ is the proportion of selecting $(x,y)$), $A_\lambda=\sum_{x,y} \lambda(x,y) (x-y) (x-y)^\top$ is the design matrix and $\Delta (x^*, y, x',y') = (x^*-y-(x'-y'))^\top\theta$.
    \label{thm1}
\end{restatable}
Our proof (\Cref{app:lower_bound}) employs a proof strategy akin to the lower bounds for standard linear bandits \cite{soare2014best, fiez2019sequential, lindner2022interactively}. The approach involves analyzing the log-likelihood ratio of a sequence of observations (that correspond to the sequence of selected arms and adversarial actions) made in two minimally distinct robust bandit instances and exploring how to achieve a distinct solution with a minimal log-likelihood ratio. In contrast to the standard linear bandit case, we have a minimax optimization problem that we avoid explicitly solving. Instead, we solve a regular convex optimization problem for each $y’ \in \cY(x')$ separately and obtain the final sample complexity guarantee by maximizing over different solutions.\looseness=-1

We would like to highlight that in the robust setting, the change-of-measure argument results in a constrained non-convex optimization problem, contrasted with the standard setting where it results in a constrained convex optimization problem. To address this challenge, we introduce several relaxations to transform the constrained non-convex optimization problem into a constrained convex one, facilitating its solution. In Section~\ref{sec:algorithms}, we derive an upper bound on the sample complexity of an adaptive algorithm, inspired by our lower bound. The order optimality of this algorithm's sample complexity demonstrates the tightness of our lower bound, affirming that the relaxations used to obtain the lower bound do not compromise its tightness.

The result obtained in \Cref{eq:lowerbound-eq}, although similar to the lower bounds presented in~\cite{soare2014best, fiez2019sequential}, 
is distinct from them in two aspects. First, it involves $\max_{y\in\cY(x^*)}$ and $\min_{y'\in\cY(x)}$. This step embodies the robustness requirement that, $\forall x'\in\cX\setminus x^*$ and $\forall y\in\cY(x^*)$, at least one $y'\in\cY(x')$ exists, where the mean reward of $x'-y'$ is lower than the mean reward of $x^*-y$. Second, the gaps are truncated with $0$. 
We note that $\Delta (x^*, y, x',y') = (x^*-y-(x'-y'))^\top\theta$ can be negative in general. Therefore, we truncate it with $0$ to avoid considering arms with a negative gap. By definition of $x^*$, there is at least one $y' \in \cY(x')$ such that $\Delta(x^*,y, x',y')>0$ (for every $y$ and $x'$), thus the denominator in \Cref{eq:lowerbound-eq} is always larger than $0$ (since $y'$ is chosen to minimize the ratio). Moreover, we note that, for any fixed $x\in\cX\setminus x^*$, $\min_{y\in\cY(x^*)}\max_{y'\in\cY(x)}\Delta(x^*,y,x',y')=\Delta_r(x^*,x)$, indicating that the gap in \Cref{eq:lowerbound-eq} would be equal to the robust gap, if the numerator (representing the uncertainty in the estimate) was set equal for all arms. That is reflected in a worst-case bound on the right hand side presented in Proposition~\ref{prop2.2}. 



We also note the following relation to the non-robust (standard) BAI problem. In case $\Y(x)$ is the same singleton set for every $x \in \X$, then it holds $\Delta (x^*, y,x',y') = \Delta(x^*,x)$ for every $x \in \X$ where $\Delta(x^*,x):= (x^* -x)^\top\theta$. In such a case, RBAI reduces to the standard best-arm identification problem and the obtained lower bound in \Cref{eq:lowerbound-eq} reduces to the BAI lower bound (see \cite[Theorem 3.1]{soare2015sequential}).\looseness=-1

We aim to determine the sample complexity of algorithms for solving the RBAI problem. In order to do so, we will utilize the lower bound that we have derived to define the best sample complexity any algorithm can achieve.\looseness=-1
\begin{definition} [RBAI Sample Complexity] \label{def:sample complexity} 
The sample complexity of an RBAI instance $\nu$ is defined as: 
    \begin{equation}
        H_{\mathrm{R}}(\nu)=\min _\lambda \max_{y \in \mathcal{Y}(x^*)}\max _{x'\in \cX\setminus x^*}\min_{y'\in\cY(x')}\frac{\|x^* - y - (x' - y') \|^2_{A_{\lambda}^{-1}}}{\max\{\Delta (x^*, y, x',y'),0\}^{2}}.
    \end{equation}
\end{definition}
To have a point of comparison, it is also useful to establish a worst-case upper bound on $H_{\mathrm{R}}(\nu)$, which is provided in terms of the smallest robust gap by the following proposition:
\begin{restatable}{proposition}{worstcaselb}
    For any RBAI problem $\nu$, we have $H_{\mathrm{R}}(\nu) \leq 4d / \min_{x \in \cX \setminus \lbrace x^* \rbrace}\Delta_r(x^*,x)^2$, where $\Delta_r(x^*,x)$ is the robust gap from \Cref{eq:robust_gap}.
    \label{RBAI lb worst}
\end{restatable}

We also utilize the oracle allocation strategy to gain a better understanding of the lower bound. In particular, the oracle strategy selects arms based on the design that minimizes $H_{\mathrm{R}}$, i.e., 
\begin{equation}\label{eq:oracle}
     \lambda^{\star} \in \underset{\lambda}{\operatorname{argmin}} \max_{y \in \mathcal{Y}(x^*)}\max _{x'\in \cX\setminus x^*}\min_{y'\in\cY(x')}\frac{\|x^* - y - (x' - y') \|^2_{A_{\lambda}^{-1}}}{\max\{\Delta (x^*, y, x',y'),0\}^{2}}.
\end{equation}

We demonstrate in \Cref{oracle-all} that the sample complexity of this oracle strategy matches the previously established lower bound. However, this strategy is not feasible as it necessitates knowledge of the unknown parameter $\theta$. In the next section, we propose algorithms that alleviate this issue.

%% file: 4-algorithms.tex
\vspace{-2pt}
\section{Algorithms}
\vspace{-2pt}
\label{sec:algorithms}
Our algorithms rely on constructing high-probability confidence intervals for the reward function based on noisy observations. To achieve this, we briefly review how to construct such confidence intervals from observations with sub-Gaussian noise. 
We use $\mathbf{x}_n,\mathbf{y}_n$ to denote a sequence of queried points $\lbrace (x_i, y_i) \rbrace_{i=1}^n$, and $\mathbf{z}_n = \mathbf{x}_n-\mathbf{y}_n$. 
We denote the ordinary least squares estimate based on the previously collected $n$ rewards as $\hat{\theta}_n = (\sum_{i=1}^n z_i z_i^\top )^{-1} (\sum_{i=1}^n z_i r_i)$, and we use $A_{\mathbf{z}_{n}} = \sum_{i=1}^n z_i z_i^\top$. Moreover, we also use the following notation $\| v \|^2_{M} := v^\top M v$ for a positive semi-definite matrix $M$. To establish valid confidence bounds, we employ the following well-known result, which holds in the non-adaptive case where decisions are independent of noise realizations.
\begin{lemma}
    Let $\hat{\theta}_n$ be the least-squares estimator obtained using the observed rewards coming from a fixed sequence $\mathbf{z}_{n} = \mathbf{x}_n-\mathbf{y}_n$. Then, the following holds
    \begin{align}
        \mathbb{P}\Big(\forall n \in \mathbb{N}, \forall z &\in \mathcal{Z},\left|z^{\top} \theta-z^{\top} \hat{\theta}_{n}\right| \leq \|z\|_{A_{\mathbf{z}_{n}}^{-1}} \sqrt{2\log \left( {| \mathcal{Z} |}  / \delta\right)}\Big) \geq 1-\delta.
    \end{align}
    \label{prop2.2}
\end{lemma}
To circumvent the inversion of singular matrices, we assume that $\cZ$ spans $\mathbb{R}^d$. This assumption is non-restrictive. If the span of $\cZ$ has a lower rank than $d$, we can employ an alternative basis, where all but the rank of $\cZ$ coordinates are always zero, and subsequently exclude them from the analysis.
Next, we design our first strategy for sampling in the RBAI problem.

\subsection{Robust Static Allocation Algorithm}
We begin by discussing the empirical \emph{stopping condition} for finding the best-robust arm, and then propose a static allocation rule based on it.
\looseness=-1

\textbf{Stopping condition.} 
We consider $\mathcal{C}(x')$ as the set of parameters $\theta'$ for which optimal robust arm is $x'$ and define a high-probability confidence set $\widehat{\mathcal{S}}(\mathbf{z}_n)$ centered at the estimated parameter $\hat{\theta}_n$ with $\mathbb{P}(\theta \in \widehat{\mathcal{S}}\left(\mathbf{z}_n\right)) \geq 1-\delta$. The aim is then to shrink $\widehat{\mathcal{S}}(\mathbf{z}_n)$ within $\mathcal{C}(x')$, and since $\theta \in \widehat{\mathcal{S}}(\mathbf{z}_n)$ w.h.p., we have that $\theta \in \mathcal{C}(x')$ and then $x' = x^*$. So with empirical gap $\widehat{\Delta}_n\left(x,y, x^{\prime},y'\right) = (x-y-(x'-y'))^\top\hat{\theta}_n$, the empirical stopping condition is\looseness=-1
\begin{align}
    \exists x \in \mathcal{X},&\forall y \in \mathcal{Y}(x), \forall x^{\prime} \in \mathcal{X}, \exists y' \in \mathcal{Y}(x'), \forall \theta' \in \widehat{\mathcal{S}}\left(\mathbf{z}_n\right),v^\top \theta' \geq 0 \Leftrightarrow  v^\top (\hat{\theta}_n - \theta') \leq \widehat{\Delta}_n\left(x,y, x^{\prime},y'\right),
    \label{eq:9}
\end{align}
where $v = x-y-(x'-y')$.
According to \Cref{prop2.2}, for any $\mathbf{z}_n$ we construct the following empirical confidence set from the idea of \Cref{eq:9},  \looseness=-1
    \begin{align}
        \widehat{\mathcal{S}}\left(\mathbf{z}_n\right)= 
        &\Big\lbrace\theta' \in \mathbb{R}^{d} \; \mathrm{ s.t. } \nonumber \exists x \in \mathcal{X},\forall y \in \mathcal{Y}(x), \forall x^{\prime} \in \mathcal{X} \setminus \lbrace x \rbrace, \exists y' \in \mathcal{Y}(x'), v^{\top}\left(\hat{\theta}_n-\theta'\right) \leq \|v\|_{A_{\mathbf{z}_{n}}^{-1}} \sqrt{2\log\left(|\mathcal{Z}|^{2} / \delta\right)}\Big\rbrace.
    \end{align}
 If $\widehat{\mathcal{S}}(\mathbf{z}_n)$ falls within $\mathcal{C}(x')$ for any $x' \in \cX$, the algorithm can stop and output the optimal arm $\Pi(\hat{\theta}_{n}):=\argmax_{x\in\cX}\min_{y\in\cY(x)}(x-y)^{\top}\hat{\theta}_n$. Hence, a feasible stopping condition can be written as:
 \begin{align}
     \exists x \in \mathcal{X},\forall y \in \mathcal{Y}(x), \forall x^{\prime} \in \mathcal{X}, \exists y' \in \mathcal{Y}(x'), \|x-y-(x'-y')\|_{A_{\mathbf{z}_{n}}^{-1}} \sqrt{2\log\big(\tfrac{|\mathcal{Z}|^{2}}{\delta}\big) }\leq \widehat{\Delta}_n\left(x,y, x^{\prime},y'\right).
     \label{eq:11}
 \end{align}
\textbf{Static allocation strategy.} To approach the stopping condition at each iteration, we need to devise a strategy to improve the estimation of $\hat{\theta}_n$ and reduce uncertainty. A natural approach is to choose the pair $z_n$ that our model is most uncertain about, which is commonly referred to as \emph{static} $G$-allocation. For any given $n$, the arms are selected according to the following $G$-allocation strategy:

\begin{equation} 
    \lambda^{G} \in \underset{\lambda}{\operatorname{argmin}} \max _{x\in \mathcal{X},y\in \mathcal{Y}(x)}\|x-y\|_{A_{\lambda}^{-1}},
    \label{eq:G-all}
\end{equation}
subject to $\lambda^{G} \in \lbrace \lambda \in \mathbb{R}^{|\cZ|}: \sum_{z \in \cZ} \lambda_z = 1, \lambda_z \geq 0 \rbrace$.

To implement such a static allocation algorithm, it is essential to round an allocation value $\lambda$ into a finite sequence of pairs represented by $z_1, \ldots, z_n$, and this requires a rounding procedure. There are efficient rounding procedures available in the experimental design literature that given $\varepsilon > 0$, the procedure can produce $(1+\varepsilon)$--approximate solution. The condition to achieve this is that $n$ should be larger than some minimum number of samples $r(\varepsilon)$. In this particular case, we make use of a standard rounding procedure from \citet[Chapter~12]{pukelsheim2006optimal} and $r(\varepsilon)=2\|\lambda\|_0/\varepsilon$ where $\varepsilon$ should be thought as a constant \citep{fiez2019sequential}.\looseness=-1

\begin{restatable}{theorem}{Gallocationsamplecomplexity}
\label{thm:static}
    If the G-allocation strategy is implemented with an $\varepsilon$-approximate rounding strategy and the stopping condition in \Cref{eq:11} is used, then
    \begin{equation}\nonumber
        \mathbb{P}\Big[N^{G} \leq \frac{32  d(1+\varepsilon) \log (|\cZ|^2 / \delta)}{\min_{x \in \cX \setminus \lbrace x^* \rbrace}\Delta_r(x^*,x)^2} \;\wedge \;\Pi\big(\hat{\theta}_{N^{G}}\big)=x^{*}\Big] \geq 1-\delta.
    \end{equation}
    where $\Delta_r(x^*,x)$ is the robust value gap.
\end{restatable}
We prove the sample complexity for $G$-allocation algorithm in Appendix \ref{static-proof}. The sample complexity of the proposed $G$-allocation strategy matches the worst-case optimal sample complexity for the RBAI problem as shown in \Cref{RBAI lb worst}. 




\begin{algorithm}[!t]
    \caption{Robust RAGE}\label{alg:robust_algo}
    \begin{algorithmic}[1]

    \REQUIRE Arms $\mathcal{X}$, adversary space $\{\mathcal{Y}(x)\}_{x \in \mathcal{X}}$, confidence $\delta \in(0,1)$, rounding approximation~factor $\varepsilon$
    \STATE \textbf{Initialization}: $t=1$; $\widehat{\mathcal{X}}_{1}=\mathcal{X} $
    \WHILE{$| \widehat{\mathcal{X}}_{t}| > 1$}
        \STATE   $\delta_t \leftarrow \frac{\delta}{t^2}$
        \STATE   $\lambda_t^* \leftarrow \arg \min _{\lambda } \max_{x \in \widehat{\mathcal{X}}_t}\max_{y \in \mathcal{Y}(x)}\max_{x' \in \widehat{\mathcal{X}}_t}\min_{y' \in \mathcal{Y}(x')} \|x-y-(x'-y')\|^2_{A^{-1}_{\lambda}} $
        \STATE   $\rho(\widehat{\mathcal{X}}_t) \leftarrow \min _{\lambda } \max_{x \in \widehat{\mathcal{X}}_t}\max_{y \in \mathcal{Y}(x)}\max_{x' \in \widehat{\mathcal{X}}_t}\min_{y' \in \mathcal{Y}(x')} \|x-y-(x'-y')\|^2_{A^{-1}_{\lambda}} $
        \STATE   $N_t \leftarrow \max \left\{\left\lceil 2^{(2t+1)} \rho(\widehat{\mathcal{X}}_t)(1+\varepsilon) \log \left(|\mathcal{Z}|^2 / \delta_t\right)\right\rceil, r(\varepsilon)\right\} $
        \STATE  $(z_1, \ldots, z_{N_t}) \leftarrow \textsc{Round}\left(\lambda_t^*, N_t\right) $
        \STATE  Pull arms $z_1, \ldots, z_{N_t}$ and obtain $r_1$, $\ldots, r_{N_t} $
        \STATE  Compute $\hat{\theta}_t=A_t^{-1} b_t$ using $A_t:=\sum_{j=1}^{N_t} z_j z_j^{\top}$ and $b_t:=\sum_{j=1}^{N_t} z_j r_j $
        \STATE  $\begin{aligned}
        \widehat{\mathcal{X}}_{t+1} \leftarrow \widehat{\mathcal{X}}_t \backslash\{x \in \widehat{\mathcal{X}}_t ~\text{s.t.}~ &\exists x^{\prime} \in \widehat{\mathcal{X}}_t, \forall y' \in \mathcal{Y}(x'), \exists y \in \mathcal{Y}(x): \\
        &\|x-y-(x'-y')\|_{A^{-1}_{\lambda}}\sqrt{2\log \left(|\mathcal{Z}|^2 / \delta_t\right)}<\left(x^{\prime}-y'-(x-y)\right)^{\top} \hat{\theta}_t\}
        \end{aligned}$
        \STATE  $t \leftarrow t+1$
    \ENDWHILE
    \STATE \textbf{Return:} arm $\Pi(\hat{\theta}_{t})$

    \end{algorithmic}
\end{algorithm}

\vspace{-2pt}
\subsection{Robust Adaptive Allocation Algorithm}
\vspace{-2pt}

Leveraging the knowledge of the lower bound provided in Theorem~\ref{thm1}, we design an algorithm called Robust RAGE with a sample complexity that matches the order of the lower bound. 
The algorithm is inspired by Randomized Adaptive Gap Elimination in~\cite{fiez2019sequential}, that is an arm elimination algorithm originally introduced for standard bandits. 
The core idea involves maintaining a set $\widehat{\cX}_t$, $t=1,2,\cdots$ of candidate best arms initialized at $\widehat{\cX}_1=\cX$.
This set is iteratively pruned using confidence intervals until the best robust arm is identified. In each iteration $t$, an arm allocation design $\lambda^{*}_t$ is (based on the lower bound) set as:
\begin{equation}
    \lambda^{*}_t = \arg \min _{\lambda} \max_{x \in \widehat{\mathcal{X}}_t}\max_{y \in \mathcal{Y}(x)}\max_{x' \in \widehat{\mathcal{X}}_t}\min_{y' \in \mathcal{Y}(x')} \|x-y-(x'-y')\|^2_{A^{-1}_{\lambda}}.
    \label{eq:eq14}
\end{equation}
Let $\rho(\widehat{\cX}_t)$ represent the minimum value of the right hand side with $\lambda_t^*$. 
We here limit the set of arms to the set of candidate best arms $\widehat{\cX}_t$ in iteration $t$. The allocation design $\lambda^*_t$ is then scaled and rounded properly to obtain the arm allocation in the iteration. In particular, let 
\begin{equation}\label{eq:Nt}
    N_t = \max \Big\{\big\lceil 2^{(2t+1)} \rho(\widehat{\mathcal{X}}_t)(1+\varepsilon) \log \left(|\mathcal{Z}|^2 / \delta_t\right)\big\rceil, r(\varepsilon)\Big\}. 
\end{equation}
In the expression above, $\varepsilon$ and $r(\varepsilon)$ serve as the parameters of the rounding procedure, and $\delta_t=\frac{\delta}{t^2}$, for some $\delta\in(0,1)$, determines the confidence interval at iteration $t$. We show that this choice of $\delta_t$, guarantees that the sample complexity holds with probability at least $1-\delta$. Moreover, the allocation is carefully scaled using a $2^{2t+1}$ factor, which balances the tradeoff between the sample complexity within each iteration and the total number of iterations.  

At the end of iteration $t$, the set of candidate best arms is updated by removing the arms which are unlikely to be the optimal arm. Specifically a $1-\delta_t$ confidence interval is used to remove all arms $x\in\widehat{\cX}_t$ for which there exists an arm $x'\in\widehat{\cX}_t$ with a higher robust mean than $x$ according to the confidence intervals: 
\begin{equation}
\begin{aligned}
\widehat{\mathcal{X}}_{t+1} =\widehat{\mathcal{X}}_t \backslash\{x \in \widehat{\mathcal{X}}_t ~\text{s.t.}~ &\exists x^{\prime} \in \widehat{\mathcal{X}}_t, \forall y' \in \mathcal{Y}(x'), \exists y \in \mathcal{Y}(x): \\
        &\|x-y-(x'-y')\|_{A^{-1}_{\lambda}}\sqrt{2\log \left(|\mathcal{Z}|^2 / \delta_t\right)}<\left(x^{\prime}-y'-(x-y)\right)^{\top} \hat{\theta}_t\}.
    \end{aligned}
\end{equation}
A detailed pseudocode is provided in Algorithm~\ref{alg:robust_algo}. 
Next, we provide an upper bound on the sample complexity of the algorithm.


\begin{restatable}{theorem}{robustsamplecomplexity} \label{thm:robust_rage}
    Assume Algorithm~\ref{alg:robust_algo} is implemented with an $\varepsilon$-approximate rounding strategy. Then, after $N$ samples the algorithm returns an optimal arm with probability at least $1-\delta$, and we have:\begin{align}
    N 
    \leq 128 \log \left(\frac{|\mathcal{Z}|^2 \bar{t}^2}{\delta^2}\right)(1+\varepsilon) \bar{t} H_{\mathrm{R}}+\bar{t}(1+r(\varepsilon)),
    \end{align}
    where $\bar{t}=\left\lceil \log _2 \left(4/\min_{x \in \cX \setminus \lbrace x^* \rbrace}\Delta_r(x^*,x)\right)\right\rceil$ and $H_{\mathrm{R}}$ is from~\Cref{def:sample complexity}.
\end{restatable}

Compared to the lower bound presented in Theorem~\ref{thm1}, the sample complexity of Robust RAGE is instance optimal up to absolute constants. 
The proof employs \Cref{prop2.2} to guarantee the efficient removal of sub-optimal arms without eliminating the optimal robust arm.  We show that the sample complexity during each round remains within constant factors of the lower bound. Moreover, we show that the total number of rounds is bounded by a constant only depending on $\min_{x \in \cX \setminus \lbrace x^* \rbrace}\Delta_r(x^*,x)$. A detailed proof is provided in~\Cref{sec:adapt proof}.\looseness=-1

\looseness=-1

%% file: 6-experiment.tex
In this section, we present simulations for the linear robust bandit problem. We compare our proposed algorithm with the oracle (Oracle) strategy that knows $\theta$ and samples according to $\lambda^*$ (\Cref{eq:oracle}). We run both the static allocation strategy (Robust Static; \Cref{eq:G-all}) and Robust RAGE (Algorithm~\ref{alg:robust_algo}) until the best-robust arm is found and report on their sample complexity. We run every algorithm at a confidence level of $\delta = 0.05$. Similarly to \citet{fiez2019sequential}, to compute the allocation strategy of Robust RAGE, we used the Frank-Wolfe algorithm to find robust $\lambda_t$, followed by the rounding procedure. The noise in the observations was generated from a standard normal distribution.
\looseness=-1

\vspace{-1ex}
\subsection{Synthetic experiments}
\vspace{-1ex}
\textbf{Experiment 1: Irrelevant dimensions.} 
We first perform a synthetic experiment with a number of irrelevant dimensions to learn the robust best arm. This experiment is similar to the one in \citet{soare2015sequential,lindner2022interactively}, that have been widely performed in the linear bandit setting. We design a problem with $d+1$ arms, each with $d$ dimensional features. Each arm has associated $n_y = 5$ adversarial actions. For arms $x_1,x_2,...,x_{d}$, we have $x_i = \mathbf{e}_i$ and $x_{d+1} = \mathbf{e}_{1} + \sin(0.01)\mathbf{e}_{2}$. Here, $\mathbf{e}_i$ denotes the $i$-th unit vector. The truly used reward parameter is $\theta = 2\mathbf{e}_1$. The adversarial action space $\mathcal{Y}(x_i)$ for each arm $x_i$ (for $i=1,\dots,d$) is the same and contains $y_j= 0.01j\mathbf{e}_i$ for $j=1,..., n_y$. However, for arm $x_{d+1}$, we use $y_j(x_{d+1})= 0.01j\mathbf{e}_{1} + 1-\varepsilon$. We set $\varepsilon = \cos(0.01)$ and the best robust arm $x^*=x_1$ since $(x_1 - y_{5}(x_1))^\top\theta = 1.90 > (x_{d+1} - y_{5}(x_{d+1}))^\top\theta = 2\varepsilon - 0.1 \approx 1.8999$. \looseness=-1

\textbf{Experiment 2: Unit Sphere.} The second experiment we perform is the unit sphere experiment \cite{lindner2022interactively,tao2018best}. We sample arms $x_1,...,x_n$ uniformly at random from the surface of a $d$-dimensional unit sphere. In the standard (non-robust) unit sphere experiment, the two arms $x_i,x_j$ with the smallest gap are used to construct $\theta = x_i + 0.01(x_j-x_i)$. This can make the experiment more challenging, as the learner must be able to distinguish between $x_i$ and $x_j$. 

We modify this setup to fit our robust setting. Since we have adversarial action $y$ to consider, the performance of arm $x$ depends on the worst-performing $y$. To construct $\theta$ (based on $x$ and $y$), we first sample $x_1,...,x_n$ uniformly from a $10$-dimensional unit sphere, select the two closest arms $x_i$ and $x_j$, and set $\theta=x_i$. We let the arm $x_i$ and $x_j$ contain only one adversarial action $y_1(x_i) = - \alpha x_i$ and $y_1(x_j) = - \alpha x_j$ to maintain the complexity of the problem. For other arms, we sample $n_y = 5$ adversarial actions $y$ from a unit sphere and multiply them by a factor $\alpha=0.05$. Under this setup, we ensure that the best robust arm $x^*$ is $x_i$. This follows by noting that for every $x$, we have  $\min_{y \in \mathcal{Y}(x)}(x-y)^\top\theta \leq \| x-y  \|_2 \|\theta\|_2 \leq 1+\alpha = (x_i + \alpha x_i)^\top\theta$. \looseness=-1

\textbf{Results.} \Cref{fig:1} illustrates the results of our experiments on synthetic robust best-arm instances. Although all algorithms can identify the robust best-arm, their sample efficiency varies significantly. As anticipated, the oracle strategy exhibits the highest sample efficiency, while Robust-RAGE comes close to oracle performance (\Cref{fig:1} (b)) and outperforms Robust Static allocation in all scenarios. For example, when we increase the number of irrelevant dimensions in the first experiment, Robust Static allocation requires more samples to identify the relevant dimension. In contrast, Robust-RAGE quickly focuses on the relevant dimension and does not require additional samples when irrelevant dimensions are added to the problem. \looseness=-1

\begin{figure*}[t!]
    \centering  
    \subfigure[Irrelevant Dimensions]{ 
    \begin{minipage}{0.45\textwidth}
    \centering    
    \includegraphics[scale=0.45]{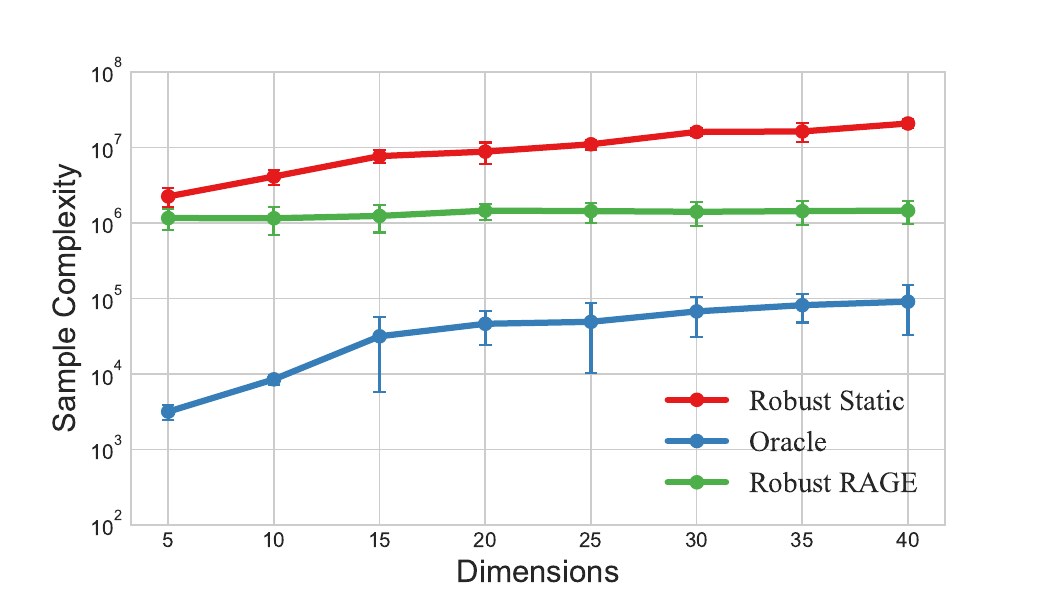} 
    \end{minipage}
    }
    \subfigure[Unit Sphere]{ 
    \begin{minipage}{0.45\textwidth}
    \centering   
        \includegraphics[scale=0.45]{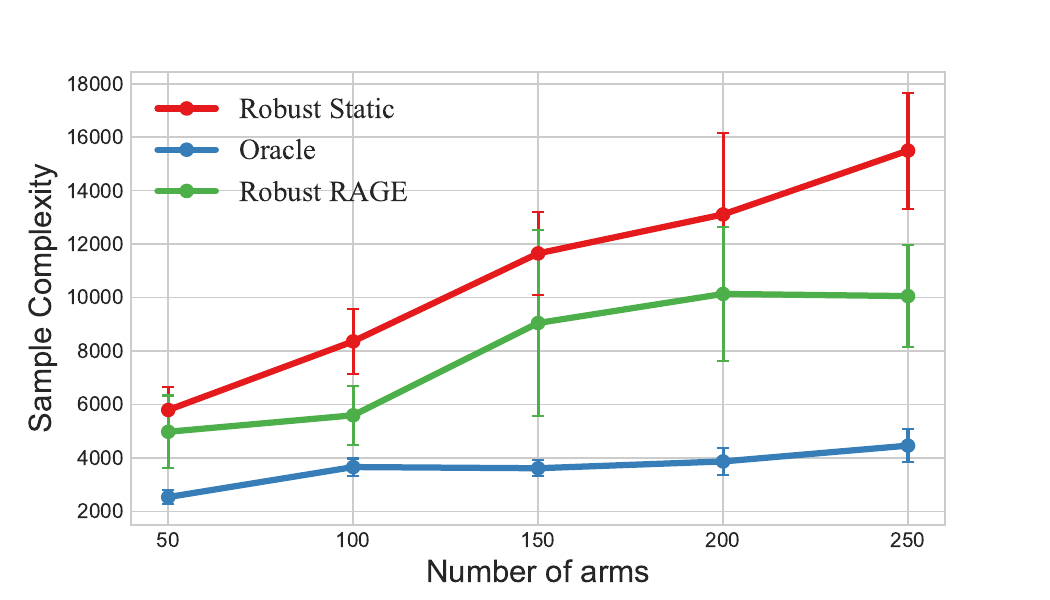}
        
    \end{minipage}
    }
    \caption{(Synthetic Experiments) The sample complexity of the proposed algorithms required to identify the best-robust arm. The results shown are averaged over 20 runs.}    
    \label{fig:1}    
    
\end{figure*}
\begin{table*}[t!]
    \centering 
    \begin{tabular}{ccccc}
    \hline
    \multirow{2}{*}{Method} & Risk Index                     & Hypoglycemia               & Hyperglycemia                & Euglycemia                   \\
                            & $\downarrow$               & (\%) $\downarrow$          & (\%) $\downarrow$            & (\%) $\uparrow$              \\ \hline
    Calculator              & 4.07 $\pm$ 3.49          & 3.48 $\pm$ 7.33          & 11.23 $\pm$ 13.95           & 85.29 $\pm$ 14.30          \\
    {StableOPT Cal.}              & 4.01 $\pm$ 3.43          & 1.69 $\pm$ 4.27         & 12.10 $\pm$ 12.41         & 86.20 $\pm$ 12.97        \\
    RBAI-tuned Calculator   & \textbf{3.77 $\pm$ 3.27} & \textbf{1.66 $\pm$ 4.02} & \textbf{11.18 $\pm$ 12.08} & \textbf{87.17 $\pm$ 12.78} \\ \hline
    \end{tabular}
    \caption{ Bolus insulin dose selection problem: The performance of the proposed calculator improves when instantiated with features of the discovered best-robust solutions.}
    \label{table:1}  
\end{table*}


\begin{table*}[t!]
    \centering 
    \begin{tabular}{ccccc}
        \hline
        \multirow{2}{*}{Method} & Risk                       & Hypoglycemia               & Hyperglycemia                & Euglycemia                   \\
                                & $\downarrow$               & (\%) $\downarrow$          & (\%) $\downarrow$            & (\%) $\uparrow$              \\ \hline
        Calculator              & 7.25 $\pm$ 4.24          & 17.86 $\pm$ 21.35        & 12.62 $\pm$ 7.74           & 69.52 $\pm$ 20.58          \\
        {StableOPT Cal.}              & 3.77 $\pm$ 3.13          & 2.64 $\pm$ 7.32          & 10.82 $\pm$ 11.96          & 86.52 $\pm$ 13.90          \\
        RBAI-tuned Calculator   & \textbf{3.60 $\pm$ 2.98} & \textbf{2.57 $\pm$ 7.07} & \textbf{10.62 $\pm$ 12.22} & \textbf{86.81 $\pm$ 14.10} \\ \hline
\end{tabular}
    \caption{Worst-case bolus insulin dose selection problem: The proposed calculator maintains the best performance across all the evaluated metrics under the worst-case scenario evaluation.}
    \label{table:2}
\end{table*}

\vspace{-1ex}
\subsection{Robust Best Dose Identification}
\vspace{-1ex}
Basal-bolus therapy is a treatment strategy commonly used in the management of diabetes, particularly type one diabetes. It involves long-acting basal insulin and after-meal rapid-acting bolus insulin. In this experiment, we aim to learn a \emph{bolus recommendation} model for each individual patient taking into account their characteristics. There are two important factors in bolus insulin calculation. The \emph{carbohydrate factor (CarbF)} measures how far an individual's blood glucose will rise per unit of carbohydrates, and the \emph{correction factor (CorrF)} quantifies how far an individual’s blood glucose will fall per unit of insulin. In order to ensure a safe blood glucose level, it is crucial to have accurate information regarding two key factors \cite{walsh2011guidelines}. The correction factor should be appropriately adjusted considering the patient's average glucose readings, while the carbohydrate factor is associated with the patient's weight. However, estimating both factors can be prone to inaccuracies. To provide reliable guidance on determining the appropriate bolus insulin dosage for patients consuming low-carb and high-carb meals, we aim to identify these factors in a robust manner.\looseness=-1

\textbf{Simulator.} We use UVA/Padova model \citep{kovatchev2009silico} to simulate the glucoregulatory system and an open-source version of the simulator \citep{Xue2018Simglucose}. It is a widely used simulator in many advanced bolus recommendation algorithms \citep{lee2020toward, zhu2020insulin, demirel2021safe}. 
By inputting meal events (meal amount and meal time) along with basal and bolus insulin amounts, the simulator can simulate the meal intakes and insulin infusion events and generate a comprehensive simulation of the patient's blood glucose level history throughout the specified time period. The simulator contains $30$ virtual patients across ages (10 children, 10 adolescents, and 10 adults).

\textbf{Experiment Setup.} We calculate the lowest negative Magni risk score \citep{magni2007model} across the patient's simulated blood glucose (BG) history as the reward for each simulation. Magni risk function provides a non-linear mapping between BG and risk value, where low BG levels have significantly faster growth in risk compared to high BG levels. We consider both low-carb (lc) and high-carb (hc) meal scenarios and amounts of these meals are calculated based on the patient's BMR rate (more details can be found in the \Cref{sec:realexp}).
 Since the \emph{total daily insulin dose} (TDI) is used to calculate CarbF and CorrF, we include it in our arm features. We formulate our experiment as an RBAI problem considering patient-specific parameter $\theta$ with arms $x_i= [ \textup{lc}, \textup{hc}, \textup{TDI}_i, \textup{CarbF}_i, \textup{CorrF}_i ]$ and their corresponding adversarial actions $y_{j,x_i} = [0,0, 0, \varepsilon ^{\textup{CarbF}_i}_j, \varepsilon ^{\textup{CorrF}_i}_j]$.
Here, $\textup{TDI}_i$ is sampled from original TDI ($\textup{TDI}_i \in \left [ 0.75,1.25 \right ] * \textup{TDI}_{\textup{original}}$) provided by simulator and $\textup{CarbF}_i$ and $\textup{CorrF}_i$ are calculated according to the constants suggested in \citet{fox2020deep}. Since these factors are only approximately estimated, we consider the worst-case and construct adversarial actions for the two factors by sampling them uniformly within suggested intervals of clinical trials \citep{walsh2011guidelines} so that $\textup{CarbF}_i - \varepsilon ^{\textup{CarbF}_i}_j \in [0.95,1.05] * 500 /  \textup{TDI}_i$ and $\textup{CorrF}_i - \varepsilon ^{\textup{CorrF}_i}_j \in \left [ 1500/  \textup{TDI}_i, 2200/  \textup{TDI}_i \right ]$. We input our arm features and their adversarial actions into the bolus calculator and the simulator injects the dosage amount from the calculator whenever a meal event occurs. We perform simulations for low-carb and high-carb scenarios, and the reward $r_t$ for RBAI problem is the sum of two rewards.

\textbf{Performance Metrics.} To evaluate the performance of the proposed Robust RAGE tuned calculator, we select four standard glycemic metrics \citep{maahs2016outcome}. Two primary goals are: keeping blood glucose levels within a target range and minimising the occurrence of hypoglycemia. Our metrics include percentage time in the glucose target range of $[70,180] \mathrm{mg} / \mathrm{dL}$ (Euglycemia), percentage time below $70 \mathrm{mg} / \mathrm{dL}$ (Hypoglycemia), percentage time above $180 \mathrm{mg} / \mathrm{dL}$ (Hyperglycemia) and a sum of low and high blood glycemic index (Risk Index). All metrics are evaluated based on 20 independent 5-day simulations with low-carb and high-carb meal events for each patient. Results are expressed by mean values and standard deviations.

\textbf{Results.} We performed the experiment by simulating 15 arms, each containing 5 adversarial actions. The baseline calculator, referred to as "Calculator", utilizes TDI, CarbF, and CorrF, from the tuned simulator configuration in \cite{fox2020deep}. Additionally, we implement another robust algorithm StableOPT \citep{bogunovic2018adversarially} with a linear kernel for comparison. We compare both baselines to the "Robust RAGE tuned Calculator", which determines these parameters based on the features of the found best-robust configuration. \looseness=-1

As illustrated in \Cref{table:1}, the metrics of RBAI-tuned calculator demonstrate superior performance compared to the two baselines. Notably, there is a significant decrease in the percentage of time spent in Hypoglycemia events, accompanied by an improvement in the amount of time patients spend within a safe BG range. The experiment utilizes the best-robust configuration identified for tuning the calculator.\looseness=-1

In \Cref{table:2}, we conduct a \emph{worst-case} evaluation, where we perturb both the default configuration of the Calculator and the discovered robust configuration obtained through Robust RAGE and StableOPT. Our objective is to minimize the reward within the set of considered perturbations. We then compare their robust performance using the same metrics as before. This experiment highlights the importance of selecting a robust arm. Comparing the results to \Cref{table:1}, it is evident that the baseline performance (Calculator) experiences a significant unsafe drop. This indicates that the normal Calculator's configuration is highly non-robust. Conversely, our RBAI-tuned calculator exhibits only a negligible decrease in performance while maintaining a satisfactory percentage of time in Euglycemia. Moreover, our findings show that the worst-case performance of StableOPT falls between that of the Calculator and Robust RAGE. This suggests that StableOPT is unable to identify the most robust arm with the same sample size as Robust RAGE. Robust RAGE, with its theoretical guarantees, outperforms StableOPT by identifying the best robust arm.

%% file: 7A-Appendix-A.tex
\newpage
\onecolumn
\appendix

\section{Lower Bound Proofs}

\subsection{Lower Bound Proof} \label{lower bound proof}
\label{app:lower_bound}

    

\lowerbound*

\begin{proof}
Without the loss of generality, we use $x_1$ and $x_2$ to represent the best-robust arms for the environment with parameter $\theta$ and $\theta'$ respectively. We consider the following inequalities similar to the proof in \citet{kaufmann2016complexity}. Define $\mathcal{A}$ a $\delta$-PAC algorithm and let $A$ be the event that algorithm $\mathcal{A}$ recommends $x_1$ as the best robust arm. Here we use the binary relative entropy as the distance function $d(x,y) = x\log(x/y)+(1-x)\log((1-x)/(1-y))$ and we have the following statements hold true:

\begin{equation}
    \mathbb{P}_\nu\left [ A \right ] \geq 1-\delta ,
    \label{eq4.2}
\end{equation}
\begin{equation}
    \mathbb{P}_{\nu'}\left [ A \right ] \leq \delta,
    \label{eq4.3}
\end{equation}
\begin{equation}
d\left(\mathbb{P}_{\nu}[A], \mathbb{P}_{\nu^{\prime}}[A]\right) \geq \log (1 / 2 \delta).
\label{eq4.4}
\end{equation}

\Cref{eq4.2} and \Cref{eq4.3} follow directly from the definition of PAC algorithms and the choice of event $A$. We thus have that the probability of event $A$ in environment $\nu$ (denoted $\mathbb{P}_\nu\left [ A \right ]$) is higher than $1-\delta$, where $x^*_{\nu}=x_1$ is the true best robust arm of the environment. Also, the probability of event $A$ in environment $\nu'$ (denoted $\mathbb{P}_{\nu'}\left [ A \right ]$) is smaller than $\delta$ since the best robust arm is different (i.e. $x^*_{\nu'}=x_2 \neq x_1$). We also introduce the following helping lemma. \\

\begin{proposition} [Lemma 1 in \cite{kaufmann2016complexity}]
    Let $N_{z_i}(t)$ denote the number of draws of pairs $z_i=x_i-y_i$ up to round $t$ and suppose $t$ is the stopping time of an algorithm $\mathcal{A}$. Also, let $\nu$ and $\nu'$ be two bandit models and $A$ an event such that $0<\mathbb{P}_\nu\left [ A \right ]<1$. Then,
    
    \begin{equation}
        \sum_{i=1}^{K} \mathbb{E}_{\nu}\left[N_{z_{i}}(t)\right] K L\left(\nu, \nu^{\prime}\right) \geq d\left(\mathbb{P}_{\nu}[A], \mathbb{P}_{\nu^{\prime}}[A]\right).
    \end{equation}
    \label{pac condition}
\end{proposition}

Now we introduce $ \tilde{\varepsilon} = \theta'-\theta$ and assume a static algorithm $\mathcal{A}$ that performs the fixed sequence of pulls $\mathbf{z}_t$. Let $\mathbf{z}_t(\mathcal{A}) =(z_1,...,z_t)$ be the sequence of $t$ arms and adversarial actions selected by the strategy $\mathcal{A}$ and let $(r_1,...,r_t)$ be the corresponding observed rewards, where $r_i=(x_i-y_i)^\top \theta + \eta _i$ with $\eta_i \sim \mathcal{N}(0,1)$. Then, we can introduce the log-likelihood ratio of the observations up to time $t$ under algorithm $\mathcal{A}$:

\begin{align}
L_{t}(r_{1}, \ldots, r_{t})&=\log \left(\prod_{s=1}^{t} \frac{\mathbb{P}_{\nu}\left(r_{s} \mid x_{s},y_s\right)}{\mathbb{P}_{\nu^{\prime}}\left(r_{s} \mid x_{s},y_s\right)}\right) \\
&=\sum_{s=1}^{t} \log \left(\frac{\mathbb{P}\left((x_{s}-y_s)^{\top} \theta+\eta_{s} \mid x_{s},y_s\right)}{\mathbb{P}\left((x_{s}-y_s)^{\top} \theta^{\prime}+\eta_{s}^{\prime} \mid x_{s},y_s\right)}\right) \\
&=\sum_{s=1}^{t} \log \left(\frac{\mathbb{P}\left(\eta_{s}\right)}{\mathbb{P}\left(\eta_{s}^{\prime}\right)}\right) \\
&=\sum_{s=1}^{t} \log \left(\frac{\exp \left(-\eta_{s}^{2} / 2\right)}{\exp \left(-\eta_{s}^{\prime 2} / 2\right)}\right) \quad\left(\text { since both } \eta \text { and } \eta^{\prime} \sim \mathcal{N}(0,1)\right) \\
&=\sum_{s=1}^{t} \log \left(\exp \left(-\eta_{s}^{2} / 2+\eta_{s}^{\prime 2} / 2\right)\right)=\sum_{s=1}^{t} \frac{\left(r_{s}-(x_{s}-y_s)^{\top} \theta^{\prime}\right)^{2}-\left(r_{s}-(x_{s}-y_s)^{\top} \theta\right)^{2}}{2} \\
&=\sum_{s=1}^{t} \frac{r_{s}^{2}-2 r_{s} (x_{s}-y_s)^{\top} \theta^{\prime}+\left((x_{s}-y_s)^{\top} \theta^{\prime}\right)^{2}-r_{s}^{2}+2 r_{s} (x_{s}-y_s)^{\top} \theta-\left((x_{s}-y_s)^{\top} \theta\right)^{2}}{2} \\
&=\sum_{s=1}^{t} \frac{2 r_{s} (x_{s}-y_s)^{\top}\left(\theta-\theta^{\prime}\right)+\left((x_{s}-y_s)^{\top} \theta^{\prime}-(x_{s}-y_s)^{\top} \theta\right)\left((x_{s}-y_s)^{\top} (\theta^{\prime}+\theta)\right)}{2} \\
&=\sum_{s=1}^{t}\left((x_{s}-y_s)^{\top} \tilde{\varepsilon}\right) \frac{-2 r_{s}+2 (x_{s}-y_s)^{\top} \theta+(x_{s}-y_s)^{\top} \tilde{\varepsilon}}{2} \\
&=\sum_{s=1}^{t}\left((x_{s}-y_s)^{\top} \tilde{\varepsilon}\right)\left(\frac{-2 \eta_{s}+(x_{s}-y_s)^{\top} \tilde{\varepsilon}}{2}\right).
\end{align}

After a simple rewriting, we obtain
\begin{equation}
\begin{aligned}
\mathbb{E}_{\nu}\left[L_{t}\right] &=\mathbb{E}_{\nu}\left[\sum_{s=1}^{t}\left((x_{s}-y_s)^{\top} \tilde{\varepsilon}\right)\left(\frac{-2 \eta_{s}+(x_{s}-y_s)^{\top} \tilde{\varepsilon}}{2}\right)\right] \\
&=\frac{1}{2} \mathbb{E}_{\nu}\left[\sum_{s=1}^{t}\left((x_{s}-y_s)^{\top} \tilde{\varepsilon}\right)^{2}\right] - \underbrace{\mathbb{E}_{\nu}[\eta_{s}]}_{=0} \\
&=\frac{1}{2} \mathbb{E}\left[\sum_{s=1}^{t} \tilde{\varepsilon}^{\top} (x_{s}-y_s) (x_{s}-y_s)^{\top} \tilde{\varepsilon}\right] \\
&=\frac{1}{2} \mathbb{E}\left[\tilde{\varepsilon}^{\top} A_{\mathbf{z}_{t}} \tilde{\varepsilon}\right],
\end{aligned}
\end{equation}

where $A_{\mathbf{z}_{t}}=\sum_{s=1}^{t} z_s z_s^{\top} = \sum_{s=1}^{t} (x_{s}-y_s) (x_{s}-y_s)^{\top}$ is the design matrix corresponding to the fixed sequence $\mathbf{z}_t$.

\paragraph{Soft allocation.} We denote the notion of soft-allocation design $\lambda \in \mathbb{R}^{|\mathcal{Z}|}$, which is the proportions of pulls to the pairs $z \in \mathcal{Z}$. Now the design matrix $A_{\mathbf{z}_t}$ for design $\lambda$ is the matrix $A_{\lambda}=\sum_{x \in \mathcal{X}, y \in \mathcal{Y}(x)}\lambda(x,y) (x-y) (x-y)^{\top}$. From an allocation $\mathbf{z}_t$ we can derive the corresponding soft design $\lambda_{\mathbf{z}_t}$ as $\lambda_{\mathbf{z}_t}(z_i) = N_{i,t}/t$, where $N_{i,t}$ denotes the number of times combination $z_i$ is selected in $\mathbf{z}_t$. Then the design matrix becomes $A_{\mathbf{z}_t} = tA_{\lambda}$. Given a random stopping time $\tau$, the allocation of pulls of the arm becomes:

\begin{equation}
\lambda=\left[\frac{\mathbb{E}\left[N_{1, \tau}\right]}{\mathbb{E}[\tau]} \quad \cdots \frac{\mathbb{E}\left[N_{|\mathcal{Z}|, \tau}\right]}{\mathbb{E}[\tau]}\right]^{\top},
\end{equation}
where $\mathbb{E}\left[N_{i,\tau}\right]$ denotes the expected number of pulls of combination $z_i$ up to round $\tau$ with $\sum_{i=1}^{|\mathcal{Z}|} N_{i,\tau}=\tau$.

Thus for the soft allocation $\lambda$, we have 
\begin{equation}
    \mathbb{E}_{\nu}\left[L_{t}\right]=\frac{1}{2} \mathbb{E}\left[\tilde{\varepsilon}^{\top} \tau A_{\lambda} \tilde{\varepsilon}\right]=\frac{1}{2} \tilde{\varepsilon}^{\top} \mathbb{E}\left[\tau A_{\lambda}\right] \tilde{\varepsilon}=\frac{1}{2} \mathbb{E}[\tau]\left(\tilde{\varepsilon}^{\top} A_{\lambda} \tilde{\varepsilon}\right).
    \label{eq:26}
\end{equation}

With $\delta$-PAC condition stated in \cref{pac condition} and \Cref{eq:26}, we obtain the same result for the expectation of the stopping time.
\begin{equation}
    \frac{1}{2} \mathbb{E}[\tau]\left(\tilde{\varepsilon}^{\top} A_{\lambda} \tilde{\varepsilon}\right) \geq \log (1 / 2 \delta) \Leftrightarrow \mathbb{E}[\tau] \geq 2 \log (1 / 2 \delta) \frac{1}{\tilde{\varepsilon}^{\top} A_{\lambda} \tilde{\varepsilon}}.
    \label{eq:27}
\end{equation}

\paragraph{Lower bound.} To get the lower bound on $\tau$, we consider $\varepsilon $ that maximizes the lower bound in \Cref{eq:27}. The difference of the best arm in the above two environments implies that there exists at least one $x' \in \mathcal{X}$ such that

\begin{equation}
    \min_{y \in \mathcal{Y}(x^*)}(x^*-y)^\top\theta'-\min_{y' \in \mathcal{Y}(x')}(x'-y')^\top\theta'<0.
    \label{eq:28}
\end{equation}

Let us define $\Delta(x,y, x',y')=(x-y-(x'-y'))^\top\theta$. Then, \Cref{eq:28} can be used as a constraint to figure out the lower bound of the stopping time. But to avoid explicit minimisation in \Cref{eq:28}, we rewrite \Cref{eq:28} as the constraint in the following minimization problem:

\begin{align}\nonumber
&\min_{\varepsilon} \: \frac{1}{2}\varepsilon^{\top} A_{\lambda}\varepsilon\\
&\text{s.t.:} \: \exists y\in \cY(x^*), \exists x' \in \mathcal{X}, \forall y'\in\Y(x'):  v^\top\theta'<0,\label{eq:29}
\end{align}
where $v=x^*-y-(x'-y')$. Adding $v^\top\theta$ to both sides, the constraint is equal to:
\begin{equation}
    v^\top\theta'<0 \Leftrightarrow v^\top \varepsilon>v^\top\theta\Leftrightarrow v^\top \varepsilon > \Delta(x^*,y, x',y').
\end{equation}

We can separate the set of constraints by fixing $y$ and $x'$ and create the following minimization problem:

\begin{align}
\nonumber
&\min_\varepsilon \: \frac{1}{2}\varepsilon^{\top} A_{\lambda}\varepsilon\\
&\text{s.t.:} \:  \text{for fixed}~ y\in\cY(x^*),x' \in\cX\setminus \{x^*\}, \forall y'\in\Y(x'),:   v^\top\varepsilon > \Delta(x^*,y, x',y').\label{eq:31} 
\end{align}

We then take minimum over all $y\in\cY(x^*)$ and $x'\in \cX\setminus x^*$ from the solution to the optimization problem in \Cref{eq:31} to get our final solution for \Cref{eq:29}. In addition, we define the following optimization problem by fixing the constraint corresponding to $y'$, 

\begin{align}
\nonumber
&\min_\varepsilon \: \frac{1}{2}\varepsilon^{\top} A_{\lambda}\varepsilon\\
&\text{s.t.:} \:  \text{for fixed}~ y\in\cY(x^*),x' \in\cX\setminus \{x^*\}, y'\in\Y(x'):   v^\top\varepsilon > \Delta(x^*,y, x',y').\label{eq:fixedyp} 
\end{align}

The solution to optimization problem~\eqref{eq:31} is larger than the maximum over $y'\in \cY(x')$ for the solutions to the optimization problem~\eqref{eq:fixedyp}. The rationale is that the minimum of $\varepsilon^{\top}A_{\lambda}\varepsilon$ under all constraints for all $y'\in\cY(x')$ will be larger than when some constraints are removed.

This optimization problem~\eqref{eq:fixedyp} is a regular convex optimization problem that can be solved using the Lagrangian multiplier method:

\begin{eqnarray}
L(\varepsilon, \gamma)= \frac{1}{2}\varepsilon^{\top} A_{\lambda} \varepsilon+\gamma\left(-v^{\top} \varepsilon+\Delta(x^*,y, x',y')+\alpha\right).
\end{eqnarray}

Here, $\gamma>0$ is Lagrangian multiplier and $\alpha>0$ is a slack variable. The corresponding derivatives are given as follows:

\begin{equation}
\begin{aligned}
&\frac{\partial L}{\partial \varepsilon}=A_{\lambda} \varepsilon-\gamma v_{y'}=0 , \\
&\frac{\partial L}{\partial \gamma}=-v_{y'}^{\top} \varepsilon+\Delta(x^*,y, x',y')+\alpha=0 .
\end{aligned}
\end{equation}

where $v_{y'} = x^*-y - (x'-y')$ with fixed $y$ and $x'$. Then, we obtain

\begin{equation}
\begin{aligned}
&A_{\lambda} \varepsilon= \gamma v_{y'}\Leftrightarrow A_{\lambda}^{1 / 2} \varepsilon= \gamma A_{\lambda}^{-1 / 2}  v_{y'} , \\
&v_{y'}^{\top} \varepsilon=\Delta(x^*,y, x',y')+\alpha .
\end{aligned}
\end{equation}

Then, we have that


\begin{align} 
& v_{y'}^\top \varepsilon=v_{y'}^\top A_\lambda^{-\frac{1}{2}} A_\lambda^{\frac{1}{2}} \varepsilon=\gamma\|v_{y'}\|^2_{A_{\lambda}^{-1}}, \\ 
& v_{y'}^\top \varepsilon=v_{y'}^\top A_\lambda^{-\frac{1}{2}} A_\lambda^{\frac{1}{2}} \varepsilon=\frac{1}{\gamma}\|\varepsilon\|^2_{A_{\lambda}}, \\
&\gamma v_{y'}^\top \varepsilon = \gamma v_{y'}^\top A_\lambda^{-\frac{1}{2}} A_\lambda^{\frac{1}{2}}\varepsilon=\|v_{y'}\|_{A_{\lambda}^{-1}}\|\varepsilon\|_{A_{\lambda}}.\label{eq:37}
\end{align}

Based on \eqref{eq:37}, it follows that


\begin{align}
\|\varepsilon\|_{A_{\lambda}} &\geq \frac{v_{y'}^\top \varepsilon}{\|v_{y'}\|_{A_{\lambda}^{-1}}}\\
& = \frac{\Delta(x^*,y, x',y')+\alpha}{\|v_{y'}\|_{A_{\lambda}^{-1}}}
\end{align}



So, for every $y'$, allowing $\alpha$ to be $0$ 

\begin{equation}
\|\varepsilon\|_{A_{\lambda}} \geq \frac{\Delta(x^*,y, x',y') }{\|v_{y'}\|_{A_{\lambda}^{-1}}}
\end{equation}

Thus, 

\begin{equation}
\|\varepsilon\|^2_{A_{\lambda}} \geq \max_{y'\in\cY(x')}\frac{(\max\{\Delta(x^*,y, x',y'), 0\})^2 }{\|v_{y'}\|^2_{A_{\lambda}^{-1}}}. \label{eq:44}
\end{equation}

Note that all $y'$ with $\Delta(x^*, y, x', y')<0$ can be dropped without affecting the right-hand side. By definition of $x^*$, there is at least one $y'$ such that $\Delta(x^*,y, x',y')>0$. Thus the numerator on the right-hand side is always larger than $0$. Also, given the assumption of the uniqueness of $z = x - y$, the norm in the denominator cannot be $0$.

The lower bound stated in \eqref{eq:44} is for
 a fixed $y \in \mathcal{Y}(x^*)$ and $x'\in\cX\setminus x^*$. As stated earlier we can obtain a lower bound on~\eqref{eq:29}, by taking minimum over this solution for all $y$ and $x'$,

\begin{equation}
     \|\varepsilon\|^2_{A_{\lambda}} \geq \min_{y \in \mathcal{Y}(x^*)}\min _{x' \in \cX\setminus x^*}\max_{y'\in\cY(x')}\left(\frac{\max\{\Delta (x^*, y,x',y'),0\} }{\|v_{y'}\|_{A_{\lambda}^{-1}}}\right)^2. \label{eq:41}
\end{equation}

Combining the result in \eqref{eq:41}, we get the eventual lower bound

\begin{equation}
    \mathbb{E}[\tau] \geq C_\delta \max_{y \in \mathcal{Y}(x^*)}\max _{x' \in \cX\setminus x^*}\min_{y'\in\cY(x')}\frac{\|x^* - y - (x' - y')\|^2_{A_{\lambda}^{-1}}}{\max\{\Delta (x^*, y,x',y'),0\}^{2}}
    \label{lowerbound}
\end{equation}
where $C_\delta = 2 \log (1 / 2 \delta)$.

\end{proof}

\subsection{Proof of Proposition \ref{RBAI lb worst}}
\label{sec: RBAI lb worst}


\worstcaselb*

\begin{proof}
We proceed to bound the sample complexity parameter: 

\begin{align}
    H_{\mathrm{R}}(\nu)&=\min _\lambda \max_{y \in \mathcal{Y}(x^*)}\max _{x'\in \mathcal{X}\setminus x^*}\min_{y'\in\mathcal{Y}(x')}\frac{\|x^* - y - (x' - y') \|^2_{A_{\lambda}^{-1}}}{\max\{\Delta (x^*, y, x',y'),0\}^{2}} \\
    &\overset{(a)}{\leq} \min _\lambda \max_{y \in \mathcal{Y}(x^*)}\max _{x'\in \mathcal{X}\setminus x^*}\min_{y'\in\mathcal{Y}(x')}\frac{4\max _{z \in \mathcal{Z}}\|z\|_{A^{-1}_{\lambda}}^2}{\max\{\Delta (x^*, y, x',y'),0\}^{2}}  \\
    &= \min _\lambda \frac{4\max _{z \in \mathcal{Z}}\|z\|_{A^{-1}_{\lambda}}^2}{\min_{y \in \mathcal{Y}(x^*)}\min _{x'\in \mathcal{X}\setminus x^*}\max_{y'\in\mathcal{Y}(x')}\max\{\Delta (x^*, y, x',y'),0\}^{2}}   \\
    &\overset{(b)}{=}\frac{4 \min _{\lambda } \max _{z \in \mathcal{Z}}\|z\|_{A^{-1}_{\lambda}}^2}{\min _{x'\in \mathcal{X}\setminus x^*} \Delta_r(x^*, x')^2} \\
    &\overset{(c)}{\leq} \frac{4d}{\min _{x'\in \mathcal{X}\setminus x^*} \Delta_r(x^*, x')^{2}} \label{eq:57}
\end{align}
where 
$(a)$ follows using triangle inequality on $z^* = x^* - y, z' = x'-y'$, $(b)$ follows 

\begin{align}
    & \min_{y \in \mathcal{Y}(x^*)}\min _{x'\in \mathcal{X}\setminus x^*}\max_{y' \in \mathcal{Y}(x')}\max\{\Delta (x^*, y, x',y'),0\}\\
    &=\min _{x'\in \mathcal{X}\setminus x^*}\left (\min_{y \in \mathcal{Y}(x^*)}\max_{y' \in \mathcal{Y}(x')}\max\{\Delta (x^*, y, x',y'),0\}  \right )\\
    &=\min _{x'\in \mathcal{X}\setminus x^*}\left ( \max\lbrace \min_{y \in \mathcal{Y}(x^*)}(x^*-y)^\top\theta - \min_{y' \in \mathcal{Y}(x')}(x'-y')^\top\theta , 0\rbrace \right ) \\
    &=\min _{x'\in \mathcal{X}\setminus x^*} \Delta_r(x^*, x'),
\end{align}



 and $(c)$ the last inequality follows the well-known Kiefer-Wolfowitz equivalence Theorem in \cite{kiefer1960equivalence}. Equality holds in \Cref{eq:57}, for example, if all $z=x-y$ are linearly independent and have the same gap value $\min _{x'\in \mathcal{X}\setminus x^*} \Delta_r(x^*, x')$.

\end{proof}

\input{7B-Appendix-B}
\input{7C-Appendix-C}

%% file: 7B-Appendix-B.tex
\section{Oracle Arm Selection Strategy} 
\label{oracle-all}

Let $\mathcal{C}(x')$ be the set of parameters $\theta'$ for which the optimal robust arm is $x'$. An \textit{oracle} is defined by a static allocation strategy that owns the information of $\theta$. Since we assume the \textit{oracle} acknowledges the exact value of $\theta$, it also has knowledge of $\mathcal{C}(x^*)$. Our aim is to construct a consistent confidence set $\mathcal{S}^*(\mathbf{z}_n) \subseteq \mathbb{R}^d$ centered in $\theta$ such that the least-squares estimate $\hat{\theta}_n$ belongs to $\mathcal{S}^*(\mathbf{z}_n)$ with high probability:

\begin{equation}
    \mathbb{P}\left(\hat{\theta}_{n} \in \mathcal{S}^{*}\left(\mathbf{z}_{n}\right) \text { and } \mathcal{S}^{*}\left(\mathbf{z}_{n}\right) \text { is centered in } \theta\right) \geq 1-\delta \text {. }
\end{equation}

So our stopping criterion checks whether the confidence set $\mathcal{S}^*(\mathbf{z}_n)$ is contained in $\mathcal{C}(x^*)$ or not. We aim to define an allocation $\mathbf{z}_n$ which leads to $\mathcal{S}^*(\mathbf{z}_n,) \subseteq \mathcal{C}(x^*)$ as quickly as possible. The condition is equivalent to 

\begin{equation}
    \forall x' \in \mathcal{X}, \forall \theta' \in \mathcal{S}^{*}\left(\mathbf{z}_{n}\right),\min_{y\in\mathcal{Y}(x^*)}(x^*-y)^\top\theta'-\min_{y'\in\mathcal{Y}(x')}(x'-y')^\top\theta'
    \geq 0.
    \label{con:oracle}
\end{equation}

An equivalent condition is as follows:

\begin{equation}
    \forall x' \in \mathcal{X}, \exists y' \in \mathcal{Y}(x'), \forall y \in \mathcal{Y}(x^*), \forall \theta' \in \mathcal{S}^{*}\left(\mathbf{z}_{n}\right),(x^*-y)^\top\theta'-(x'-y')^\top\theta' \geq 0.
\end{equation}

Then, we add $\Delta (x^*, y,x', y') = (x^*-y)^\top\theta - (x'-y')^\top\theta$ to both sides, and obtain
\begin{equation}
    (x^*-y - (x'-y'))^\top(\theta-\theta') \leq \Delta (x^*, y,x', y').
\end{equation}




Using Lemma \ref{prop2.2}, we bound the prediction error of a fixed allocation strategy selecting $\mathbf{z}_n$. We construct the following confidence set 
    \begin{align}
        (x^*-y - (x'-y'))^{\top}\left(\theta-\theta'\right) \leq \|(x^*-y - (x'-y'))\|_{A_{\mathbf{z}_{n}}^{-1}} \sqrt{2\log\left(|\mathcal{Z}|^{2} / \delta\right)}, \label{eq:56}
    \end{align}

Now, the stopping condition $\mathcal{S}^*(\mathbf{z}_n) \subseteq \mathcal{C}(x^*)$ is equivalent to verifying that for any $y \in \mathcal{Y}(x^*)$ and $x' \in \mathcal{X}$,  there exists $y' \in \mathcal{Y}(x')$,

\begin{equation}
    \|(x^*-y - (x'-y'))\|_{A_{\mathbf{z}_{n}}^{-1}} \sqrt{2\log \left(|\mathcal{Z}|^{2} / \delta\right)} \leq \Delta (x^*, y,x',y'). \label{eq:57}
\end{equation}

So a straightforward allocation strategy is obtained after squaring both sides
\begin{align}
    \mathbf{z}_{n}^{*} &=\arg \min _{\mathbf{z}_{n}}  \max_{y \in \mathcal{Y}(x^*)}\max _{x' \in \cX\setminus x^*}\min_{y'\in\cY(x')}\frac{\|(x^*-y - (x'-y'))\|^2_{A_{\mathbf{z}_{n}}^{-1}}}{\max\{\Delta (x^*, y,x',y'),0\}^2}.
    \label{ora-all}
\end{align}

The sample complexity of the oracle is defined by the number of steps needed by the deterministic static allocation in \Cref{ora-all} to achieve the stopping condition in \Cref{eq:57}. 

We can derive the design matrix for $\mathbf{z}_{n}$ with $\lambda_{\mathbf{z}_{n}}(z)  = T_n(z)/n = \lambda^\star$, where $T_n(z)$ denote the number of times that pair $z$ is selected. Thus after $n$ queries, the design matrix is $A_{\mathbf{z}_{n}} = n A_{\lambda_{\mathbf{z}_{n}}} = n A_{\lambda^\star}$ and we have

\begin{equation}
    \max_{y \in \mathcal{Y}(x^*)}\max _{x' \in \mathcal{X}\setminus x^*}\min_{y'\in\mathcal{Y}(x')} \frac{\|x^*-y - (x'-y')\|^2_{A_{\mathbf{z}_n}^{-1}}}{\max\{\Delta (x^*, y,x',y'),0\}^2} = \frac{1}{n}\max_{y \in \mathcal{Y}(x^*)}\max _{x' \in \mathcal{X}\setminus x^*}\min_{y'\in\mathcal{Y}(x')} \frac{\|x^*-y - (x'-y')\|^2_{A_{\lambda^\star}^{-1}}}{\max\{\Delta (x^*, y,x',y'),0\}^2}.
    \label{eq:68}
\end{equation}
We denote $N^*$ as the sample complexity of the oracle and have
\begin{align}
     &\frac{1}{N^*}\min _{\lambda^\star}\max_{y \in \mathcal{Y}(x^*)}\max _{x' \in \mathcal{X}\setminus x^*}\min_{y'\in\mathcal{Y}(x')} \frac{\|x^*-y - (x'-y')\|^2_{A_{\lambda^\star}^{-1}}\left(2\log (|\mathcal{Z}|^{2} / \delta)\right)}{\max\{\Delta (x^*, y,x',y'),0\}^2}=1.
\end{align}
Then the sample complexity of the oracle becomes
\begin{align}
    N^* &= \min _{\lambda^\star}\max_{y \in \mathcal{Y}(x^*)}\max _{x' \in \mathcal{X}\setminus x^*}\min_{y'\in\mathcal{Y}(x')} \frac{\|x^*-y - (x'-y')\|^2_{A_{\lambda^\star}^{-1}}\left(2\log (|\mathcal{Z}|^{2} / \delta)\right)}{\max\{\Delta (x^*, y,x',y'),0\}^2},\\
    &=2\log \left(|\mathcal{Z}|^{2} / \delta\right)H_{\mathrm{R}}.
\end{align}
which shows the sample complexity of the oracle matches the derived lower bound in \Cref{thm1}.

%% file: 7C-Appendix-C.tex
\section{Proof of Static Allocation Complexity} \label{static-proof}


\Gallocationsamplecomplexity*

\begin{proof}
    With empirical gap $\widehat{\Delta}_n\left(x,y, x^{\prime},y'\right) = (x-y-(x'-y'))^\top\hat{\theta}_n$, we recall the stopping condition with $v = x - y - (x' - y')$ is
    \begin{align}
       \nonumber \exists x \in \mathcal{X},\forall y \in \mathcal{Y}(x), &\forall x^{\prime} \in \mathcal{X}, \exists y' \in \mathcal{Y}(x'), \\
         &\|x-y-(x'-y')\|_{A_{\mathbf{z}_{n}}^{-1}} \sqrt{2\log\left(|\mathcal{Z}|^{2} / \delta\right) }\leq \widehat{\Delta}_n\left(x,y, x^{\prime},y'\right).
    \end{align}
    Then using triangle inequality on $z = x-y$, we have
    \begin{equation}
        \left\|x-y-(x'-y')\right\|_{A_{\mathbf{z}_{n}}^{-1}} \sqrt{2\log \left(|\mathcal{Z}|^{2} / \delta\right)} \leq 2\, \max _{z \in \mathcal{Z}}\left\|z\right\|_{A_{\mathbf{z}_{n}}^{-1}} \sqrt{2\log \left(|\mathcal{Z}|^{2} / \delta\right)}.
        \label{eq:77}
    \end{equation}
     The stopping condition becomes $ \forall x' \in \mathcal{X}, \exists y' \in \mathcal{Y}(x')$,
     \begin{equation}
         2\, \max _{z \in \mathcal{Z}}\left\|z\right\|_{A_{\mathbf{z}_{n}}^{-1}} \sqrt{2\log \left(|\mathcal{Z}|^{2} / \delta\right)}  \leq \widehat{\Delta}_{n}\left(x^*, y, x',y'\right).
         \label{eq:78}
     \end{equation}
     

    From \Cref{prop2.2} we have that the following inequalities hold with probability $1-\delta$ under the condition $\forall y \in \mathcal{Y}(x^*), \forall x^{\prime} \in \mathcal{X}, \exists y' \in \mathcal{Y}(x')$ :
    \begin{equation}
        (x^*-y-(x'-y'))(\theta - \hat{\theta}_n) \leq \left\|x^* - y - (x' - y')\right\|_{A_{\mathbf{z}_{n}}^{-1}} \sqrt{2\log \left(|\mathcal{Z}|^{2} / \delta\right)}.
    \end{equation}
    Replace with term $\widehat{\Delta}_{n}\left(x^*, y, x',y'\right)$ and $\Delta \left(x^*, y, x',y'\right)$, we have
    \begin{align}
        \widehat{\Delta}_{n}\left(x^*, y, x',y'\right) &\geq \Delta \left(x^*, y, x',y'\right)-\left\|x^* - y - (x' - y')\right\|_{A_{\mathbf{z}_{n}}^{-1}} \sqrt{2\log\left(|\mathcal{Z}|^{2} / \delta\right)}  \\
        &\geq \Delta \left(x^*, y, x',y'\right)-2\, \max _{z \in \mathcal{Z}}\left\|z\right\|_{A_{\mathbf{z}_{n}}^{-1}} \sqrt{2\log \left(|\mathcal{Z}|^{2} / \delta\right)}.\label{eq:77}
    \end{align}
    Combining inequality~\eqref{eq:78} and~\eqref{eq:77}, we derive a sufficient stopping condition $\forall y \in \mathcal{Y}(x^*, \forall x^{\prime} \in \mathcal{X}, \exists y' \in \mathcal{Y}(x')$
    \begin{align}
         4\, \max _{z \in \mathcal{Z}}\left\|z\right\|_{A_{\mathbf{z}_{n}}^{-1}} \sqrt{2\log \left(|\mathcal{Z}|^{2} / \delta\right)} & \leq \Delta \left(x^*, y, x',y'\right).
    \end{align}
    Square both sides and accompany with conditions, we have
    \begin{align}
        32 \max _{z \in \mathcal{Z}}\left\|z\right\|^2_{A_{\mathbf{z}_{n}}^{-1}} \left(\log \left(|\mathcal{Z}|^{2} / \delta\right)\right) & \leq \min_{y \in \mathcal{Y}(x^*)}\min _{x'\in \mathcal{X}\setminus x^*}\max_{y' \in \mathcal{Y}(x')}\max\{\Delta (x^*, y, x',y'),0\}^2\\
         &\leq \min _{x'\in \mathcal{X}\setminus x^*} \Delta_r(x^*, x')^2 .\label{eq:80}
    \end{align}
    Then similar to the oracle strategy, after $n$ queries the design matrix is $A_{\mathbf{z}_{n}} = n A_{\lambda^G}$. To satisfy the stopping condition~\eqref{eq:80}, we denote $N^G$ as the sample complexity of $G$-allocation and 
    to avoid fractional design, we use an $\varepsilon$-approximate rounding procedure so that
    \begin{align}
        N^G &\leq  \frac{32(1+\varepsilon) \min_{\lambda^G} \max _{z \in \mathcal{Z}}\left\|z\right\|^2_{A_{\lambda^G}^{-1}} \left(\log (|\mathcal{Z}|^{2} / \delta)\right)}{\min _{x'\in \mathcal{X}\setminus x^*} \Delta_r(x^*, x')^2} \\
        &\leq \frac{32(1+\varepsilon)d \left(\log (|\mathcal{Z}|^{2} / \delta)\right)}{\min _{x'\in \mathcal{X}\setminus x^*} \Delta_r(x^*, x')^2} 
    \end{align}
    where the last inequality holds using the result in \citet{kiefer1960equivalence}. 
\end{proof}

\section{Proof of Adaptive Algorithm Complexity} \label{sec:adapt proof}

\robustsamplecomplexity*
    
    

\begin{proof}
    In the following proof, $\theta$ is defined as the unknown parameter in the RBAI problem.
    Let $\mathcal{E}_t:=\left\{\widehat{\mathcal{X}}_t \subseteq \mathcal{S}_t\right\}\bigcap \left \{ x^* \in \widehat{\mathcal{X}}_t \right \}$ be the event that the robust gap value for all the arms in $\widehat{\mathcal{X}}_t $ is smaller than $2^{-(t-2)}$ for $t$-th round. So we define
    \begin{equation}
        \mathcal{S}_t:=\left\{x' \in \mathcal{X} | \min _{y \in \mathcal{Y}\left(x^*\right)}\left(x^*-y\right)^{\top} \theta-\min _{y' \in \mathcal{Y}(x')}(x'-y')^{\top} \theta \leq 2^{-(t-2)}\right\}.
    \end{equation}
    Equivalently, we have

    \begin{equation}
        \mathcal{S}_t=\left\{x' \in \mathcal{X} | \exists y \in \mathcal{Y}(x^*), \forall y' \in \mathcal{Y}(x'), \left(x^*-y\right)^{\top} \theta-(x'-y')^{\top} \theta \leq 2^{-(t-2)}\right\}.
    \end{equation}
    
    Note that for $t>\log _2\left(4 / \min_{x \in \cX \setminus \lbrace x^* \rbrace}\Delta_r(x^*,x)\right)$, we have 
    
    \begin{align}
        \mathcal{S}_t&=\left\{x' \in \mathcal{X} | \min _{y \in \mathcal{Y}\left(x^*\right)}\left(x^*-y\right)^{\top} \theta-\min _{y' \in \mathcal{Y}(x')}(x'-y')^{\top} \theta\leq \min_{x \in \cX \setminus \lbrace x^* \rbrace}\Delta_r(x^*,x)\right\}, \\
        & = \left \{ x^* \right \}.
    \end{align}
    
    We will first show that $P\left(\mathcal{E}_1\right) \geq 1-\delta_1$ and $P\left(\mathcal{E}_t \mid \mathcal{E}_{t-1}\right) \geq 1-\delta_t$. Let $x, x' \in \widehat{\mathcal{X}}_t$ and denote $v =  x - y - (x' - y')$. With \cref{prop2.2} and the $\varepsilon$-approximate rounding strategy, it holds with probability at least $1-\delta_t$ that:
    \begin{equation}
        v^\top(\theta - \hat{\theta}_t) \leq \sqrt{2 \log \left(\frac{|\mathcal{Z}|^2}{\delta_t}\right) \frac{1+\varepsilon}{N_t}}\|v\|_{A_{t}^{-1}}
    \end{equation}

    Then accompanied with the condition on $y,y'$,
    \begin{equation}
        \exists y \in \mathcal{Y}(x), \forall y'\in\mathcal{Y}(x'), v^\top(\theta - \hat{\theta}_t) \leq \sqrt{2 \log \left(\frac{|\mathcal{Z}|^2}{\delta_t}\right) \frac{1+\varepsilon}{N_t}}\max_{y \in \mathcal{Y}(x)}\min_{y' \in \mathcal{Y}(x')}\|v\|_{A_{t}^{-1}}.
    \end{equation}
    We take the length of a round $N_t=\max \left\{\left\lceil 2^{(2t+1)} \rho(\widehat{\mathcal{X}}_t)(1+\varepsilon) \log \left(|\mathcal{Z}|^2 / \delta_t\right)\right\rceil, r(\varepsilon)\right\}$ where $\rho(\widehat{\mathcal{X}}_t) =\min _{\lambda } \max_{x \in \widehat{\mathcal{X}}_t}\max_{y \in \mathcal{Y}(x)}\max_{x' \in \widehat{\mathcal{X}}_t}\min_{y' \in \mathcal{Y}(x')} \|v\|^2_{A^{-1}_{\lambda}}$ and then select arms to reduce uncertainty in {$\widehat{\mathcal{X}}_t$}, we get
    
    \begin{align}
    &v^\top(\theta - \hat{\theta}_t) \\
    & \leq 2^{-t} \sqrt{\left(\min _{\lambda } \max_{x \in \widehat{\mathcal{X}}_t}\max_{y \in \mathcal{Y}(x)}\max_{x' \in \widehat{\mathcal{X}}_t}\min_{y' \in \mathcal{Y}(x')}\|\tilde{v}\|_{A_\lambda^{-1}}^2\right)^{-1}}\max_{y \in \mathcal{Y}(x)}\min_{y' \in \mathcal{Y}(x')}\|v\|_{A_{t}^{-1}} \\
    & \leq 2^{-t} \sqrt{\left(\min _{\lambda } \max_{x \in \widehat{\mathcal{X}}_t}\max_{y \in \mathcal{Y}(x)}\max_{x' \in \widehat{\mathcal{X}}_t}\min_{y' \in \mathcal{Y}(x')}\|\tilde{v}\|_{A_\lambda^{-1}}^2\right)^{-1}}\left(\min _{\lambda } \max_{x \in \widehat{\mathcal{X}}_t}\max_{y \in \mathcal{Y}(x)}\max_{x' \in \widehat{\mathcal{X}}_t}\min_{y' \in \mathcal{Y}(x')}\|\tilde{v}\|_{A_\lambda^{-1}}\right) \\
    &\leq 2^{-t}.
    \end{align}

    It follows that $\mathbb{P}\left(\mathcal{E}_t \mid \mathcal{E}_{t-1}\right) \geq 1-\delta_t$. 

    \begin{equation}
        \mathbb{P}\left(\exists x, x' \in \mathcal{X}_t, \exists y \in \mathcal{Y}(x), \forall y'\in\mathcal{Y}(x'), \left|v^{\top}\left(\theta-\hat{\theta}_t\right)\right|>2^{-t} \mid \mathcal{E}_{t-1}\right) \leq \delta_t.
        \label{cf 114}
    \end{equation}

    \textit{Claim 1:} Every arm $x \in \widehat{\mathcal{X}}_t$ such that 
    \begin{equation}
        \min _{y \in \mathcal{Y}\left(x^*\right)}\left(x^*-y\right)^{\top} \theta-\min _{y' \in \mathcal{Y}(x')}(x'-y')^{\top} \theta \geq 2^{-(t-1)}
    \end{equation}
    is discarded in phase $t$ so that $\widehat{\mathcal{X}}_{t+1} \subseteq \mathcal{S}_{t+1}$ with probability at least $1-\delta_t$.

    \begin{proof}
        Since we conditioned on $\mathcal{E}_{t-1}, x^* \in \widehat{\mathcal{X}}_t$. If $x' \in \mathcal{S}_{t+1}^c \cap \widehat{\mathcal{X}}_t$ then by definition 
        \begin{equation}
            \min _{y \in \mathcal{Y}\left(x^*\right)}\left(x^*-y\right)^{\top} \theta-\min _{y' \in \mathcal{Y}(x')}(x'-y')^{\top} \theta \geq 2^{-(t-1)}.
        \end{equation}
        Taking $v=x^*- y - (x'-y')$, we know . 
        \begin{equation}
            \exists y \in \mathcal{Y}(x^*), \forall y'\in\mathcal{Y}(x'), v^{\top}\theta > 2^{-(t-1)}.
        \end{equation}
        Given the number of samples $N_t$, we can ensure that the confidence bound $\exists y \in \mathcal{Y}(x^*), \forall y'\in\mathcal{Y}(x'), \|v\|^2_{A^{-1}_{t}}\sqrt{2\log \left(|\mathcal{Z}|^2 / \delta_t\right)} < 2^{-t}.$
        Based on these, we have
        \begin{align}
        v^{\top} \hat{\theta}_t & \geq v^{\top} \theta^*-\|v\|^2_{A^{-1}_{t}}\sqrt{2\log \left(|\mathcal{Z}|^2 / \delta_t\right)} \\
        &>2^{-(t-1)}-2^{-t}=2^{-t} \\
        &>\|v\|^2_{A^{-1}_{t}}\sqrt{2\log \left(|\mathcal{Z}|^2 / \delta_t\right)}.
        \end{align}
        
        However, this is precisely the condition that the algorithm eliminates $x'$.
        We next show that the best robust arm $x^*$ will not be discarded in a phase with high probability.
    \end{proof}
    \textit { Claim 2: } $x^* \in \widehat{\mathcal{X}}_{t}$ with probability at least $1-\delta_t$.
    \begin{proof}
        We prove this claim by contradiction. We know that $x^*$ is in $\widehat{\mathcal{X}}_{t-1}$ since $\mathcal{E}_{t-1}$ holds. Now, suppose that $x^*$ is discarded in phase $t$. This implies that there exists a $x' \neq x^*$ for $x' \in \widehat{\mathcal{X}}_t$ and $v=x'- y' - (x^*- y)$ such that $\forall y \in \mathcal{Y}(x^*), \exists y' \in \mathcal{Y}(x')$, $\|v\|^2_{A^{-1}_{t}}\sqrt{2\log \left(|\mathcal{Z}|^2 / \delta_t\right)}<v^{\top} \hat{\theta}_t$. However from the confidence interval in \Cref{cf 114}, $v^{\top}(\hat{\theta}_t-\theta) \leq \|v\|^2_{A^{-1}_{t}}\sqrt{2\log \left(|\mathcal{Z}|^2 / \delta_t\right)}$. Combining these we see that $v^{\top}(\hat{\theta}_t-\theta)<v^{\top} \hat{\theta}_t$, it implies $v^{\top} \theta>0$ and $x'$ is more robust than $x^*$, which is a contradiction.
    \end{proof}

    Finally, \Cref{lemma3} demonstrates that the probability of the algorithm providing the best robust solution after round $\bar{t}$ is at least $1-\delta$. So we can compute the total number of samples for the algorithm based on $\bar{t}=\left\lceil\log _2 4/\min_{x \in \cX \setminus \lbrace x^* \rbrace}\Delta_r(x^*,x)\right\rceil$.
    
    \begin{align}
        &N =\sum_{t=1}^{\bar{t}}\max \left\{\left\lceil 2^{(2t+1)} \rho(\widehat{\mathcal{X}}_t)(1+\varepsilon) \log \left(|\mathcal{Z}|^2 / \delta_t\right)\right\rceil, r(\varepsilon)\right\}  \\
        &\leq \sum_{t=1}^{\bar{t}} 2^{2 t+1} \log \left(\frac{|\mathcal{Z}|^2}{\delta_t}\right)(1+\varepsilon) \rho(\widehat{\mathcal{X}}_t)+\bar{t}(1+r(\varepsilon)) \\
        & \leq 32 \log \left(\frac{|\mathcal{Z}|^2 \bar{t}^2}{\delta^2}\right)(1+\varepsilon) \sum_{t=1}^{\bar{t}}\left(2^{t-2}\right)^2 \rho(\widehat{\mathcal{X}}_t)+\bar{t}(1+r(\varepsilon)) \\
        &=32 \log \left(\frac{|\mathcal{Z}|^2 \bar{t}^2}{\delta^2}\right)(1+\varepsilon) \sum_{t=1}^{\bar{t}}\left(2^{t-2}\right)^2 \min _{\lambda } \max_{x \in \widehat{\mathcal{X}}_t}\max_{y \in \mathcal{Y}(x)}\max_{x' \in \widehat{\mathcal{X}}_t}\min_{y' \in \mathcal{Y}(x')} \|v\|^2_{A^{-1}_{\lambda}}+\bar{t}(1+r(\varepsilon)) \\
        &\stackrel{(a)}{\leq} 128 \log \left(\frac{|\mathcal{Z}|^2 \bar{t}^2}{\delta^2}\right)(1+\varepsilon) \sum_{t=1}^{\bar{t}}(2^{t-2})^2 \min _{\lambda }\max_{y \in \mathcal{Y}(x^*)}\max_{x' \in \widehat{\mathcal{X}}_t}\min_{y' \in \mathcal{Y}(x')}\|x^*-y-(x'-y')\|_{A_\lambda^{-1}}^2+\bar{t}(1+r(\varepsilon)) \\
        &\stackrel{(b)}{\leq} 128 \log \left(\frac{|\mathcal{Z}|^2 \bar{t}^2}{\delta^2}\right)(1+\varepsilon)\bar{t}(2^{\bar{t}-2})^2 \min _{\lambda }\max_{y \in \mathcal{Y}(x^*)}\max_{x' \in \widehat{\mathcal{X}}_t}\min_{y' \in \mathcal{Y}(x')}\|x^*-y-(x'-y')\|_{A_\lambda^{-1}}^2+\bar{t}(1+r(\varepsilon)) \\
        &= 128 \log \left(\frac{|\mathcal{Z}|^2 \bar{t}^2}{\delta^2}\right)(1+\varepsilon) \bar{t} \min _{\lambda }\max_{y \in \mathcal{Y}(x^*)}\max_{x' \in \widehat{\mathcal{X}}_t}\min_{y' \in \mathcal{Y}(x')}\frac{\|x^*-y-(x'-y')\|_{A_\lambda^{-1}}^2}{(2^{-(\bar{t}-2)})^2}+\bar{t}(1+r(\varepsilon)) \\
        &\stackrel{(c)}{\leq} 128 \log \left(\frac{|\mathcal{Z}|^2 \bar{t}^2}{\delta^2}\right)(1+\varepsilon) \bar{t}\min _{\lambda }\max_{y \in \mathcal{Y}(x^*)}\max_{x' \in \widehat{\mathcal{X}}_t}\min_{y' \in \mathcal{Y}(x')}\frac{\|x^*-y-(x'-y')\|_{A_\lambda^{-1}}^2}{\max\{\Delta (x^*, y, x',y'),0\}^{2}}+\bar{t}(1+r(\varepsilon)) \\
        &= 128 \log \left(\frac{|\mathcal{Z}|^2 \bar{t}^2}{\delta^2}\right)(1+\varepsilon) \bar{t} H_{\mathrm{R}}+\bar{t}(1+r(\varepsilon)) 
    \end{align}
    where $(a)$ follows that if $v=x_i-y_i-(x_j-y_j)$, given $v = (x^*-y-(x_i-y_i)) - (x^*-y-(x_j-y_j))$, we have 
    \begin{equation}
        \max_{x \in \widehat{\mathcal{X}}_t}\max_{y \in \mathcal{Y}(x)}\max_{x' \in \widehat{\mathcal{X}}_t}\min_{y' \in \mathcal{Y}(x')} \|v\|^2_{A^{-1}_{\lambda}} \leq 4 \max_{y \in \mathcal{Y}(x^*)}\max_{x' \in \widehat{\mathcal{X}}_t}\min_{y' \in \mathcal{Y}(x')}\|x^*-y-(x'-y')\|_{A_\lambda^{-1}}^2,
    \end{equation}
    $(b)$ follows that the maximum of sum is smaller than the sum of maximum and $(c)$ follows because of the definition of $\bar{t}$.
\end{proof}

\begin{lemma}[Lemma 1 in \citet{lindner2022interactively}]
        Let $\mathcal{E}_1, \ldots, \mathcal{E}_T$ be a Markovian sequence of events such that $\mathbb{P}\left(\mathcal{E}_1\right) \geq 1-\delta_1$ and $\mathbb{P}\left(\mathcal{E}_t \mid \mathcal{E}_{t-1}\right) \geq 1-\delta_t$ for all $t=2, \ldots, T$, where $\delta_t=\delta^2 / t^2$ and $\delta \in(0,1) . \mathcal{E}_t$ is independent of other events conditioned on $\mathcal{E}_{t-1}$. Then $\mathbb{P}\left(\mathcal{E}_T\right) \geq 1-\delta$.
        \label{lemma3}
    \end{lemma}

    \begin{proof}
        \begin{equation}
            \mathbb{P}\left(\mathcal{E}_T\right)=\left(\prod_{t=2}^T \mathbb{P}\left(\mathcal{E}_t \mid \mathcal{E}_{t-1}\right)\right) \mathbb{P}\left(\mathcal{E}_1\right) \geq\left(\prod_{t=2}^{\bar{t}}\left(1-\delta_t\right)\right)\left(1-\delta_1\right) \geq \prod_{t=1}^{\infty}\left(1-\frac{\delta^2}{t^2}\right)=\frac{\sin (\pi \delta)}{\pi \delta} \geq 1-\delta
        \end{equation}
        where the last inequality holds for $0 \leq \delta \leq 1$.
    \end{proof}

\section{Synthetic experiment details}
In this section, we provide some details on the implementation of the algorithms.

\begin{itemize}
    \item Robust static allocation and oracle allocation: We first compute the optimal design for static strategy in~\eqref{eq:G-all} and oracle strategy in~\eqref{eq:oracle}. Then we run both algorithms in phases and select $\gamma^t$ samples from the allocation. We optimize the performance of the algorithms for $\gamma \in (1,2)$ and finally set $\gamma=1.1$ for oracle strategy and $\gamma=1.3$ for static strategy. Then the stopping conditions in \Cref{eq:57} and \Cref{eq:11} are used to terminate experiments for each algorithm.
    \item Robust RAGE: We run this algorithm following steps in Algorithm~\ref{alg:robust_algo} and set $\varepsilon=0.1$.
\end{itemize}

\section{Robust dose identification experiment details} \label{sec:realexp}
In this section, we provide more details on the implementation of the robust dose identification experiment.
\subsection{Meal Events}

The meal events include the meal amounts and meal intake time. We use BMR rate according to the Harrison-Benedict equation \cite{harris1919biometric} to estimate expected daily carbohydrate consumption,
\begin{equation}
    \textup{BMR }= 66.5 + (13.75 * \textup{Weight})+ (5.003 * \textup{Height}) - (6.755 * \textup{Age}).
\end{equation}
The expected daily carbohydrate consumption is $\textup{BMR }*0.45 / 4 $, where we assume 45\% of calories are from carbohydrates, 4 calories per carbohydrate. Then the low-carb meal amount is $\textup{BMR }*0.45 / 4 * 0.75$, and high-carb meal is $\textup{BMR }*0.45 / 4 * 1.25$. We separate daily carbohydrate intake into 6 potential meals: breakfast, lunch, dinner, and 3 snacks. The proportion of each meal is  $[0.250, 0.035, 0.295, 0.035, 0.350, 0.035]$ and the meal time is $[7, 9.5, 12, 15, 18, 21.5]$.

\subsection{Magni risk function}

The Magni risk function is defined as:
\begin{equation}
    \operatorname{risk}(\textup{BG})=10 *\left(c_0 * \log (\textup{BG})^{c_1}-c_2\right)^2,
\end{equation}
where $c_0 = 3.35506$, $c_1 = 0.8353$ and $c_2 = 3.7932$ in \citet{magni2007model}. To calculate the rewards for the RBAI problem, we simulate two independent one-day BG histories based on low-carb and high-carb meal events. The reward is the sum of the lowest negative risk score across these two histories.

    
\subsection{Bolus Calculator Details}

\begin{table}[!h]
    \centering
    \begin{tabular}{lcccc} 
        Patient & CarbF & CorrF & Age & TDI \\
        \hline 
        adult\#001 & 9.92 & 35.70 & 61 & 50.42 \\
        adult\#002 & 8.64 & 31.10 & 65 & 57.87 \\
        adult\#003 & 8.86 & 31.90 & 27 & 56.43 \\
        adult\#004 & 14.79 & 53.24 & 66 & 33.81 \\
        adult\#005 & 7.32 & 26.35 & 52 & 68.32 \\
        adult\#006 & 8.14 & 29.32 & 26 & 61.39 \\
        adult\#007 & 11.90 & 42.85 & 35 & 42.01 \\
        adult\#008 & 11.69 & 42.08 & 48 & 42.78 \\
        adult\#009 & 7.44 & 26.78 & 68 & 67.21 \\
        adult\#010 & 7.76 & 27.93 & 68 & 64.45 \\
        adolescent\#001 & 13.61 & 49.00 & 18 & 36.73 \\
        adolescent\#002 & 8.06 & 29.02 & 19 & 62.03 \\
        adolescent\#003 & 20.62 & 74.25 & 15 & 24.24 \\
        adolescent\#004 & 14.18 & 51.06 & 17 & 35.25 \\
        adolescent\#005 & 14.70 & 52.93 & 16 & 34.00 \\
        adolescent\#006 & 10.08 & 36.30 & 14 & 49.58 \\
        adolescent\#007 & 11.46 & 41.25 & 16 & 43.64 \\
        adolescent\#008 & 7.89 & 28.40 & 14 & 63.39 \\
        adolescent\#009 & 20.77 & 74.76 & 19 & 24.08 \\
        adolescent\#010 & 15.07 & 54.26 & 17 & 33.17 \\
        child\#001 & 28.62 & 103.02 & 9 & 17.47 \\
        child\#002 & 27.51 & 99.02 & 9 & 18.18 \\
        child\#003 & 31.21 & 112.35 & 8 & 16.02 \\
        child\#004 & 25.23 & 90.84 & 12 & 19.82 \\
        child\#005 & 12.21 & 43.97 & 10 & 40.93 \\
        child\#006 & 24.72 & 89.00 & 8 & 20.22 \\
        child\#007 & 13.81 & 49.71 & 9 & 36.21 \\
        child\#008 & 23.26 & 83.74 & 10 & 21.49 \\
        child\#009 & 28.75 & 103.48 & 7 & 17.39 \\
        child\#010 & 24.21 & 87.16 & 12 & 20.65 \\
\end{tabular}
    \caption{Patients' Parameters}
    \label{tab:patient}
\end{table}

We use the following formula to calculate our bolus insulin amount

\begin{equation}
    \textup{bolus} = \frac{\textup{meal}}{\textup{CarbF}} + \textup{I}(\textup{glucose} > 150)*\frac{\textup{currentBG}-\textup{targetBG}}{\textup{CorrF}}
\end{equation}
where $\textup{targetBG}=140$ and $\textup{I}(\cdot)$ is the indicator function. The baseline calculator computes $\textup{CarbF}$ and $\textup{CorrF}$ via formula $\textup{CarbF}=500/\textup{TDI},\textup{CorrF}=1800/\textup{TDI}$ suggested in \citep{fox2020deep}. We display the calculated patients' parameters in \Cref{tab:patient}.

\subsection{Running time details}

We use the open-source UVA/Padova model package Simglucose \citep{Xue2018Simglucose} to simulate blood glucose histories for patients. Generating 1000 BG histories for each patient consumes 3 hours. Running RBAI algorithms with $|\mathcal{Z}|=75$ takes less than 10 minutes. These experiments are conducted on a laptop equipped with an 11th Gen Intel(R) i7-11800H CPU, utilizing multiple cores for parallel processing.

\section{Computation Complexity}

We provide a brief overview of the computational complexity involved in Algorithm \ref{alg:robust_algo}. The near-optimal design is computed using Frank-Wolfe algorithm \citep{Fedorov_1972}. With a suitable initialization constant factor, an approximation of the optimal design $\lambda_t^*$ in \Cref{eq:eq14} can be computed in $O(|\mathcal{Z}|^2 d^2\log\log(d))$ operations.

%% file: main.bbl
\begin{thebibliography}{39}
\providecommand{\natexlab}[1]{#1}
\providecommand{\url}[1]{\texttt{#1}}
\expandafter\ifx\csname urlstyle\endcsname\relax
  \providecommand{\doi}[1]{doi: #1}\else
  \providecommand{\doi}{doi: \begingroup \urlstyle{rm}\Url}\fi

\bibitem[Abbasi-Yadkori et~al.(2011)Abbasi-Yadkori, P{\'a}l, and Szepesv{\'a}ri]{abbasi2011improved}
Yasin Abbasi-Yadkori, D{\'a}vid P{\'a}l, and Csaba Szepesv{\'a}ri.
\newblock Improved algorithms for linear stochastic bandits.
\newblock \emph{Advances in neural information processing systems}, 24, 2011.

\bibitem[Agrawal and Goyal(2013)]{agrawal2013thompson}
Shipra Agrawal and Navin Goyal.
\newblock Thompson sampling for contextual bandits with linear payoffs.
\newblock In \emph{International conference on machine learning}, pages 127--135. PMLR, 2013.

\bibitem[Alieva et~al.(2021)Alieva, Cutkosky, and Das]{alieva2021robust}
Ayya Alieva, Ashok Cutkosky, and Abhimanyu Das.
\newblock Robust pure exploration in linear bandits with limited budget.
\newblock In \emph{International Conference on Machine Learning}, pages 187--195. PMLR, 2021.

\bibitem[Amani et~al.(2019)Amani, Alizadeh, and Thrampoulidis]{amani2019linear}
Sanae Amani, Mahnoosh Alizadeh, and Christos Thrampoulidis.
\newblock Linear stochastic bandits under safety constraints.
\newblock \emph{Advances in Neural Information Processing Systems}, 32, 2019.

\bibitem[Audibert et~al.(2010)Audibert, Bubeck, and Munos]{audibert2010best}
Jean-Yves Audibert, S{\'e}bastien Bubeck, and R{\'e}mi Munos.
\newblock Best arm identification in multi-armed bandits.
\newblock In \emph{COLT}, pages 41--53, 2010.

\bibitem[Bogunovic et~al.(2018)Bogunovic, Scarlett, Jegelka, and Cevher]{bogunovic2018adversarially}
Ilija Bogunovic, Jonathan Scarlett, Stefanie Jegelka, and Volkan Cevher.
\newblock Adversarially robust optimization with gaussian processes.
\newblock In \emph{Conference on Neural Information Processing Systems (NeurIPS)}, 2018.

\bibitem[Bogunovic et~al.(2021)Bogunovic, Losalka, Krause, and Scarlett]{bogunovic2021stochastic}
Ilija Bogunovic, Arpan Losalka, Andreas Krause, and Jonathan Scarlett.
\newblock Stochastic linear bandits robust to adversarial attacks.
\newblock In \emph{International Conference on Artificial Intelligence and Statistics}, pages 991--999. PMLR, 2021.

\bibitem[Chen et~al.(2021)Chen, Pierson, Rose, Joshi, Ferryman, and Ghassemi]{chen2021ethical}
Irene~Y Chen, Emma Pierson, Sherri Rose, Shalmali Joshi, Kadija Ferryman, and Marzyeh Ghassemi.
\newblock Ethical machine learning in healthcare.
\newblock \emph{Annual review of biomedical data science}, 4:\penalty0 123--144, 2021.

\bibitem[Dani et~al.(2008)Dani, Hayes, and Kakade]{dani2008stochastic}
Varsha Dani, Thomas~P Hayes, and Sham~M Kakade.
\newblock Stochastic linear optimization under bandit feedback.
\newblock 2008.

\bibitem[Degenne et~al.(2020)Degenne, M{\'e}nard, Shang, and Valko]{degenne2020gamification}
R{\'e}my Degenne, Pierre M{\'e}nard, Xuedong Shang, and Michal Valko.
\newblock Gamification of pure exploration for linear bandits.
\newblock In \emph{International Conference on Machine Learning}, pages 2432--2442. PMLR, 2020.

\bibitem[Demirel et~al.(2021)Demirel, Ozdemir, and Tekin]{demirel2021safe}
Ilker Demirel, Mehmet~Ufuk Ozdemir, and Cem Tekin.
\newblock Safe linear leveling bandits.
\newblock \emph{arXiv preprint arXiv:2112.06728}, 2021.

\bibitem[Ding et~al.(2022)Ding, Hsieh, and Sharpnack]{ding2022robust}
Qin Ding, Cho-Jui Hsieh, and James Sharpnack.
\newblock Robust stochastic linear contextual bandits under adversarial attacks.
\newblock In \emph{International Conference on Artificial Intelligence and Statistics}, pages 7111--7123. PMLR, 2022.

\bibitem[Fedorov(1972)]{Fedorov_1972}
V.~V. Fedorov.
\newblock \emph{Theory of optimal experiments}.
\newblock Academic Press, 1972.

\bibitem[Fiez et~al.(2019)Fiez, Jain, Jamieson, and Ratliff]{fiez2019sequential}
Tanner Fiez, Lalit Jain, Kevin~G Jamieson, and Lillian Ratliff.
\newblock Sequential experimental design for transductive linear bandits.
\newblock \emph{Advances in neural information processing systems}, 32, 2019.

\bibitem[Fox et~al.(2020)Fox, Lee, Pop-Busui, and Wiens]{fox2020deep}
Ian Fox, Joyce Lee, Rodica Pop-Busui, and Jenna Wiens.
\newblock Deep reinforcement learning for closed-loop blood glucose control.
\newblock In \emph{Machine Learning for Healthcare Conference}, pages 508--536. PMLR, 2020.

\bibitem[Gabillon et~al.(2012)Gabillon, Ghavamzadeh, and Lazaric]{gabillon2012best}
Victor Gabillon, Mohammad Ghavamzadeh, and Alessandro Lazaric.
\newblock Best arm identification: A unified approach to fixed budget and fixed confidence.
\newblock \emph{Advances in Neural Information Processing Systems}, 25, 2012.

\bibitem[Harris and Benedict(1919)]{harris1919biometric}
James~Arthur Harris and Francis~Gano Benedict.
\newblock \emph{A biometric study of basal metabolism in man}.
\newblock Number 279. Carnegie institution of Washington, 1919.

\bibitem[He et~al.(2022)He, Zhou, Zhang, and Gu]{he2022nearly}
Jiafan He, Dongruo Zhou, Tong Zhang, and Quanquan Gu.
\newblock Nearly optimal algorithms for linear contextual bandits with adversarial corruptions.
\newblock \emph{arXiv preprint arXiv:2205.06811}, 2022.

\bibitem[Jamieson and Nowak(2014)]{jamieson2014best}
Kevin Jamieson and Robert Nowak.
\newblock Best-arm identification algorithms for multi-armed bandits in the fixed confidence setting.
\newblock In \emph{2014 48th Annual Conference on Information Sciences and Systems (CISS)}, pages 1--6. IEEE, 2014.

\bibitem[Kaufmann et~al.(2016)Kaufmann, Capp{\'e}, and Garivier]{kaufmann2016complexity}
Emilie Kaufmann, Olivier Capp{\'e}, and Aur{\'e}lien Garivier.
\newblock On the complexity of best-arm identification in multi-armed bandit models.
\newblock \emph{The Journal of Machine Learning Research}, 17\penalty0 (1):\penalty0 1--42, 2016.

\bibitem[Khezeli and Bitar(2020)]{khezeli2020safe}
Kia Khezeli and Eilyan Bitar.
\newblock Safe linear stochastic bandits.
\newblock In \emph{Proceedings of the AAAI Conference on Artificial Intelligence}, volume~34, pages 10202--10209, 2020.

\bibitem[Kiefer and Wolfowitz(1960)]{kiefer1960equivalence}
Jack Kiefer and Jacob Wolfowitz.
\newblock The equivalence of two extremum problems.
\newblock \emph{Canadian Journal of Mathematics}, 12:\penalty0 363--366, 1960.

\bibitem[Kovatchev et~al.(2009)Kovatchev, Breton, Dalla~Man, and Cobelli]{kovatchev2009silico}
Boris~P Kovatchev, Marc Breton, Chiara Dalla~Man, and Claudio Cobelli.
\newblock In silico preclinical trials: a proof of concept in closed-loop control of type 1 diabetes, 2009.

\bibitem[Lee et~al.(2020)Lee, Kim, Park, Jin, and Park]{lee2020toward}
Seunghyun Lee, Jiwon Kim, Sung~Woon Park, Sang-Man Jin, and Sung-Min Park.
\newblock Toward a fully automated artificial pancreas system using a bioinspired reinforcement learning design: In silico validation.
\newblock \emph{IEEE Journal of Biomedical and Health Informatics}, 25\penalty0 (2):\penalty0 536--546, 2020.

\bibitem[Lindner et~al.(2022)Lindner, Tschiatschek, Hofmann, and Krause]{lindner2022interactively}
David Lindner, Sebastian Tschiatschek, Katja Hofmann, and Andreas Krause.
\newblock Interactively learning preference constraints in linear bandits.
\newblock In \emph{International Conference on Machine Learning}, pages 13505--13527. PMLR, 2022.

\bibitem[Maahs et~al.(2016)Maahs, Buckingham, Castle, Cinar, Damiano, Dassau, DeVries, Doyle~III, Griffen, Haidar, et~al.]{maahs2016outcome}
David~M Maahs, Bruce~A Buckingham, Jessica~R Castle, Ali Cinar, Edward~R Damiano, Eyal Dassau, J~Hans DeVries, Francis~J Doyle~III, Steven~C Griffen, Ahmad Haidar, et~al.
\newblock Outcome measures for artificial pancreas clinical trials: a consensus report.
\newblock \emph{Diabetes care}, 39\penalty0 (7):\penalty0 1175--1179, 2016.

\bibitem[Magni et~al.(2007)Magni, Raimondo, Bossi, Dalla~Man, De~Nicolao, Kovatchev, and Cobelli]{magni2007model}
Lalo Magni, Davide~M Raimondo, Luca Bossi, Chiara Dalla~Man, Giuseppe De~Nicolao, Boris Kovatchev, and Claudio Cobelli.
\newblock Model predictive control of type 1 diabetes: an in silico trial, 2007.

\bibitem[Moradipari et~al.(2021)Moradipari, Amani, Alizadeh, and Thrampoulidis]{moradipari2021safe}
Ahmadreza Moradipari, Sanae Amani, Mahnoosh Alizadeh, and Christos Thrampoulidis.
\newblock Safe linear thompson sampling with side information.
\newblock \emph{IEEE Transactions on Signal Processing}, 69:\penalty0 3755--3767, 2021.

\bibitem[Pacchiano et~al.(2021)Pacchiano, Ghavamzadeh, Bartlett, and Jiang]{pacchiano2021stochastic}
Aldo Pacchiano, Mohammad Ghavamzadeh, Peter Bartlett, and Heinrich Jiang.
\newblock Stochastic bandits with linear constraints.
\newblock In \emph{International conference on artificial intelligence and statistics}, pages 2827--2835. PMLR, 2021.

\bibitem[Pukelsheim(2006)]{pukelsheim2006optimal}
Friedrich Pukelsheim.
\newblock \emph{Optimal design of experiments}.
\newblock SIAM, 2006.

\bibitem[Rusmevichientong and Tsitsiklis(2010)]{rusmevichientong2010linearly}
Paat Rusmevichientong and John~N Tsitsiklis.
\newblock Linearly parameterized bandits.
\newblock \emph{Mathematics of Operations Research}, 35\penalty0 (2):\penalty0 395--411, 2010.

\bibitem[Soare(2015)]{soare2015sequential}
Marta Soare.
\newblock \emph{Sequential resource allocation in linear stochastic bandits}.
\newblock PhD thesis, Universit{\'e} Lille 1-Sciences et Technologies, 2015.

\bibitem[Soare et~al.(2014)Soare, Lazaric, and Munos]{soare2014best}
Marta Soare, Alessandro Lazaric, and R{\'e}mi Munos.
\newblock Best-arm identification in linear bandits.
\newblock \emph{Advances in Neural Information Processing Systems}, 27, 2014.

\bibitem[Tao et~al.(2018)Tao, Blanco, and Zhou]{tao2018best}
Chao Tao, Sa{\'u}l Blanco, and Yuan Zhou.
\newblock Best arm identification in linear bandits with linear dimension dependency.
\newblock In \emph{International Conference on Machine Learning}, pages 4877--4886. PMLR, 2018.

\bibitem[Vayena et~al.(2018)Vayena, Blasimme, and Cohen]{vayena2018machine}
Effy Vayena, Alessandro Blasimme, and I~Glenn Cohen.
\newblock Machine learning in medicine: addressing ethical challenges.
\newblock \emph{PLoS medicine}, 15\penalty0 (11):\penalty0 e1002689, 2018.

\bibitem[Walsh et~al.(2011)Walsh, Roberts, and Bailey]{walsh2011guidelines}
John Walsh, Ruth Roberts, and Timothy Bailey.
\newblock Guidelines for optimal bolus calculator settings in adults.
\newblock \emph{Journal of diabetes science and technology}, 5\penalty0 (1):\penalty0 129--135, 2011.

\bibitem[Wang et~al.(2022)Wang, Wagenmaker, and Jamieson]{wang2022best}
Zhenlin Wang, Andrew~J Wagenmaker, and Kevin Jamieson.
\newblock Best arm identification with safety constraints.
\newblock In \emph{International Conference on Artificial Intelligence and Statistics}, pages 9114--9146. PMLR, 2022.

\bibitem[Xie(2018)]{Xue2018Simglucose}
Jinyu Xie.
\newblock Simglucose.
\newblock \url{https://github.com/jxx123/simglucose}, 2018.

\bibitem[Zhu et~al.(2020)Zhu, Li, Kuang, Herrero, and Georgiou]{zhu2020insulin}
Taiyu Zhu, Kezhi Li, Lei Kuang, Pau Herrero, and Pantelis Georgiou.
\newblock An insulin bolus advisor for type 1 diabetes using deep reinforcement learning.
\newblock \emph{Sensors}, 20\penalty0 (18):\penalty0 5058, 2020.

\end{thebibliography}
